\documentclass{article}

\usepackage{microtype}
\usepackage{graphicx}
\usepackage{subfigure}
\usepackage{booktabs} % for professional tables

\usepackage{hyperref}

\usepackage[accepted]{icml2023}

\usepackage{amsmath}
\usepackage{amssymb}
\usepackage{mathtools}
\usepackage{amsthm}

\usepackage{wrapfig,booktabs}
\usepackage{paralist}
\usepackage{bm}
\usepackage{multirow}
\usepackage{caption}

\usepackage[capitalize,noabbrev]{cleveref}
\graphicspath{{./figs/}}

\theoremstyle{plain}
\newtheorem{theorem}{Theorem}[section]

\newtheorem{corollary}[theorem]{Corollary}
\theoremstyle{definition}
\newtheorem{definition}[theorem]{Definition}

\theoremstyle{remark}

\usepackage[textsize=tiny]{todonotes}

\newcommand{\eat}[1]{}
\newcommand{\vsa}{\vspace*{-0.28cm}}
\newcommand{\vsb}{\vspace*{-0.19cm}}

\newcommand{\cbit}{\begin{compactitem}}
	\newcommand{\ceit}{\end{compactitem}}
\newcommand{\cben}{\begin{compactenum}}
	\newcommand{\ceen}{\end{compactenum}}
\newcommand{\defeq}{\overset{\mathrm{def}}{=\joinrel=}}
\DeclareMathOperator*{\argmin}{argmin}

\newcommand{\normlone}{L^1}

\icmltitlerunning{Do Not Train It: A Linear Neural Architecture Search of Graph Neural Networks}

\begin{document}

\twocolumn[
\icmltitle{
    Do Not Train It: A Linear Neural Architecture Search of Graph Neural Networks
}




\icmlsetsymbol{equal}{*}

\begin{icmlauthorlist}
    \icmlauthor{Peng Xu}{cuhk,equal}
    \icmlauthor{Lin Zhang}{idea,equal}
    \icmlauthor{Xuanzhou Liu}{sigs}
    \icmlauthor{Jiaqi Sun}{sigs}
    \icmlauthor{Yue Zhao}{cmu}
    \icmlauthor{Haiqin Yang}{idea}
    \icmlauthor{Bei Yu}{cuhk}
\end{icmlauthorlist}

\icmlaffiliation{sigs}{Shenzhen International Graduate School, Tsinghua University, China}
\icmlaffiliation{idea}{International Digital Economy Academy, China}
\icmlaffiliation{cuhk}{The Chinese University of Hong Kong, China}
\icmlaffiliation{cmu}{Carnegie Mellon University, USA}

\icmlcorrespondingauthor{Bei Yu}{byu@cse.cuhk.edu.hk}
\icmlcorrespondingauthor{Lin Zhang}{linzhang0529@gmail.com}

% \begin{icmlauthorlist}
% \icmlauthor{Firstname1 Lastname1}{equal,yyy}
% \icmlauthor{Firstname2 Lastname2}{equal,yyy,comp}
% \icmlauthor{Firstname3 Lastname3}{comp}
% \end{icmlauthorlist}



\icmlkeywords{Machine Learning, ICML}

\vskip 0.3in
]



\printAffiliationsAndNotice{\icmlEqualContribution} % otherwise use the standard text.

\global\long\def\real{\mathbb{R}}

\begin{abstract}
Neural architecture search (NAS) for Graph neural networks (GNNs), called NAS-GNNs, has achieved significant performance over manually designed GNN architectures. However, these methods inherit issues from the conventional NAS methods, such as high computational cost and optimization difficulty. More importantly, previous NAS methods have ignored the uniqueness of GNNs, where GNNs possess expressive power without training.
With the randomly-initialized weights, we can then seek the optimal architecture parameters via the sparse coding objective and derive a novel NAS-GNNs method, namely neural architecture coding (NAC). Consequently, our NAC holds a no-update scheme on GNNs and can efficiently compute in linear time. Empirical evaluations on multiple GNN benchmark datasets demonstrate that our approach leads to state-of-the-art performance, which is up to $200\times$ faster and $18.8\%$ more accurate than the strong baselines.
\end{abstract}

\vsa 
\section{Introduction} 
Remarkable progress of graph neural networks (GNNs) has boosted research in various domains, such as traffic prediction, recommender systems, etc., as summarized in~\citep{2021GNNs}. The central paradigm of GNNs is to generate node embeddings through the message-passing mechanism~\citep{gnn2020_Hamilton}, including passing, transforming, and aggregating node features across the input graph. 
Despite its effectiveness, designing GNNs requires laborious efforts to choose and tune neural architectures for different tasks and datasets~\citep{you2020design}, which limits the usability of GNNs. To automate the process, researchers have made efforts to leverage neural architecture search (NAS)~\citep{liu2019darts,zhang2021automated} 
for GNNs,
% to frame the designing of GNNs as an architecture searching problem, 
including GraphNAS~\citep{gao2020graph}, Auto-GNN~\citep{zhou2019autognn}, {PDNAS~\citep{zhao2020probabilistic}} and SANE~\citep{zhao2021search}.
In this work, we refer the problem of NAS for GNNs as \textit{NAS-GNNs}.
%, which is generally known as neural architecture search (NAS)~\citep{liu2019darts,zhang2021automated}.

% Recent work has studied how to wrap GNNs into the NAS framework, such as GraphNAS~\citep{gao2020graph}, Auto-GNN~\citep{zhou2019autognn}, and SANE~\citep{zhao2021search},  and thus enables learning an adaptive GNN architecture with less human effort.

\begin{figure}
    \centering
    % \caption{Caption}
    \label{fig:my_label}
    \centering
    \includegraphics[trim={0.2cm 0.cm 0.cm 0.2cm},clip,width=0.35\textwidth]{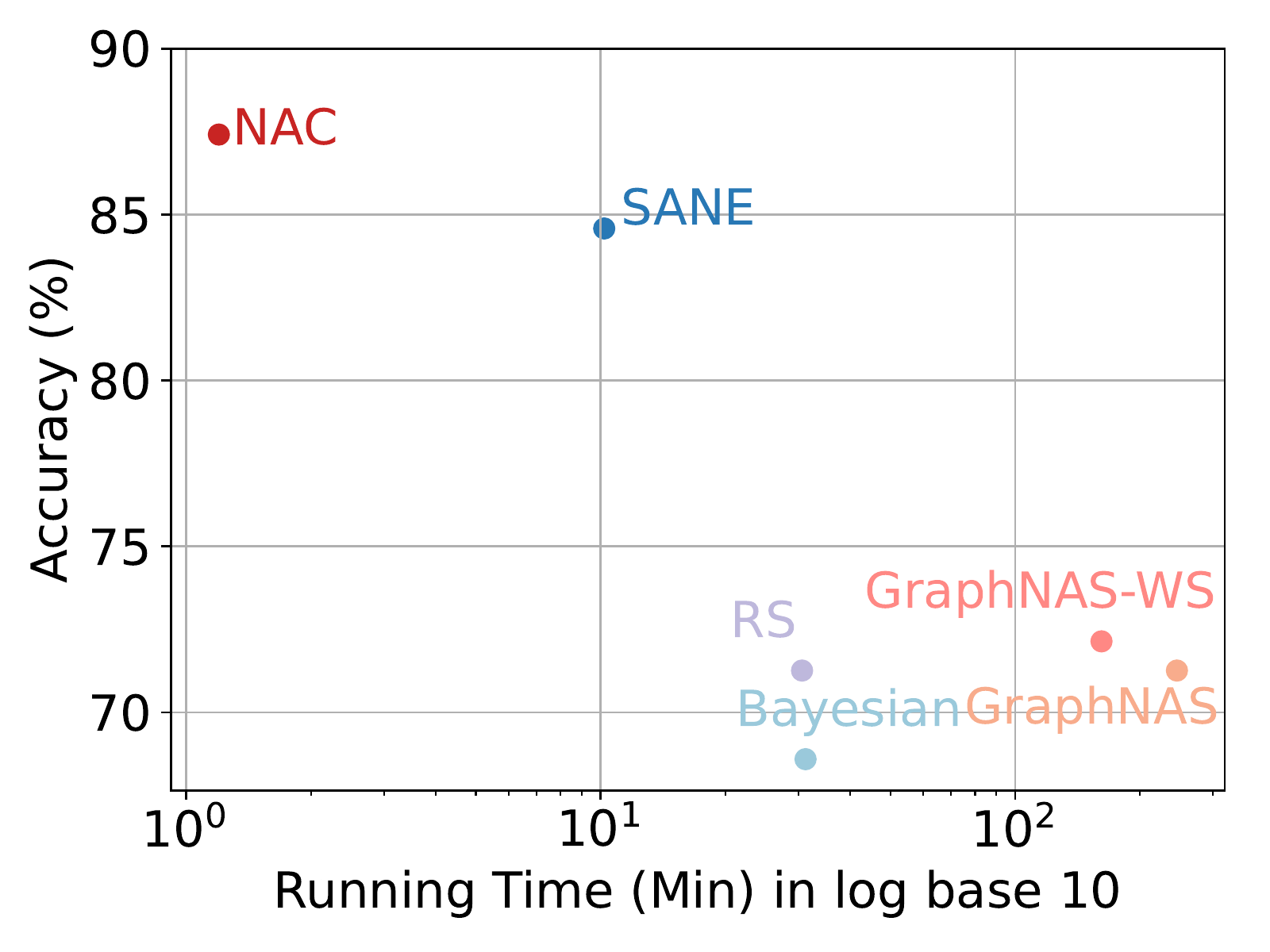}
    \caption{ \footnotesize
     Accuracy vs. running time on Cora. NAC (ours) outperforms the leading methods significantly in both accuracy and speed (in minutes).
    }
    \label{fig:overview}  \vsb
\end{figure}

While NAS-GNNs have shown promising results, they inherit issues from general NAS methods and fail to account for the unique properties of GNN operators. It is important to understand the difficulty in general NAS training (e.g., architecture searching and weight evaluation). Based on the searching strategy, NAS methods can be categorized into three types: reinforcement learning-based methods~\citep{DBLP:conf/iclr/ZophL17}, evolutionary algorithms-based methods~\citep{jozefowicz2015empirical}, and differential-based methods~\citep{liu2019darts,wu2019fbnet}. Both reinforcement learning-based and evolutionary algorithm-based methods suffer from high computational costs due to the need to re-train sampled architectures from scratch. On the contrary, the weight-sharing differential-based paradigm reuses the neural weights to reduce the search effort and produces the optimal sub-architecture directly without excessive processes, such as sampling, leading to significant computational cost reduction and becoming the new frontier of NAS.

However, the weight sharing paradigm requires the neural weights to reach {optimality} so as to obtain the optimal sub-architecture based on its bi-level optimization (BLO) strategy~\citep{liu2019darts}, which alternately optimizes the network weights (outputs of operators) and architecture parameters (importance of operators). First, it is hard to achieve the optimal neural weights in general due to the curse of dimensionality in deep learning, leading to unstable searching results, also called the optimization gap~\citep{xie2020weightsharing}. Second, this paradigm often shows a sloppy gradient estimation~\citep{bi2020stabilizing,bi2020goldnas,guo2020single} due to the alternating optimization, softmax-based estimation, and unfairness in architecture updating. This type of work suffers from slow convergence during training and is sensitive to initialization due to the wide use of early termination. If not worse, it is unclear why inheriting weights for a specific architecture is still efficient---the weight updating and sharing lack interpretability. 

\textit{Is updating the GNN weights necessary? Or, does updating weights contribute to optimal GNN architecture searching?} Existing NAS-GNN methods rely on updating the weights, and in fact, all these issues raised are due to the need to update weights to the optimum. Unlike other deep learning structures, graph neural networks behave almost linearly, so they can be simplified as linear networks while maintaining superior performance~\citep{wu2019sgc}. Inspired by this, we find that the untrained GNN model is nearly optimal in theory. \textbf{Note that the first paper on modern GNNs, i.e., GCN~\citep{kipf2016gcn}, already spotted this striking phenomenon in the experiment, but gained little attention.} To the best of our knowledge, our work is the first to unveil this and provide theoretical proof. The issues mentioned before may not be as much of a concern given no weight update is needed, making NAS-GNN much simpler.

In this paper, we formulate the NAS-GNN problem as a sparse coding problem by leveraging the untrained GNNs, called neural architecture coding (NAC). We prove that untrained GNNs have built-in orthogonality, making the output dependent on the linear output layer. With no-update scheme, we only need to optimize the architecture parameters, resulting in a single-level optimization strategy as opposed to the bi-level optimization in the weight-sharing paradigm, which reduces the computational cost significantly and improves the optimization stability. Much like the sparse coding problem~\citep{zhang2015survey}, our goal is also to learn a set of sparse coefficients for selecting operators when treating these weights collectively as a dictionary, making sharing weights straightforward and understandable. Through extensive experiments on multiple challenging benchmarks, we demonstrate that our approach is competitive with the state-of-the-art baselines,
while decreasing the computational cost significantly, as shown in \Cref{fig:overview}.

In summary, our main contributions are:{
\begin{compactitem}[-]
\item {\bf Problem Formulation}: We present (to our best knowledge) the first linear complexity NAS algorithm for GNNs, namely NAC, which is solved by sparse coding.
\item {\bf Theoretical Analysis}: Our NAC holds a no-update scheme, which  is theoretically justified by the built-in model linearity in GNNs and orthogonality in the model weights. 
\item {\bf Effectiveness and Efficiency}: We compare NAC with state-of-the-art baselines and show superior performance in both accuracy and speed.  Especially, NAC brings up to $18.8\%$ improvement in terms of accuracy and is $200\times$ faster than baselines. 
\end{compactitem}
}%\todo{change data to model weights}
\eat{\\
\noindent{\bf $\bullet$ Problem Formulation:}
We 
% show how to 
reformulate the NAS-GNN problem as a sparse coding problem,
% . Based on this, we 
proposed (to our best knowledge) {the first linear complexity NAS algorithm for GNNs}, called NAC.\\
\noindent{\bf $\bullet$ Theoretical Analysis:}
%We prove the soundness of {untrained} GNN models, providing theoretical guarantees for achieving optimal or near-optimal solution with minimal effort. This finding potentially paves a new way for NAS-GNN research. 
 We prove an untrained GNN can approximate the optimal output due to the built-in linearity and orthogonality in high dimensional data. \\
\noindent{\bf $\bullet$ Effectiveness and Efficiency:}
% Our proposed method, 
NAC shows superior performance in both accuracy and speed when compared with state-of-the-art baselines. 
NAC brings up to $18.8\%$ improvement in terms of accuracy and is $200\times$ faster than baselines. %\\
}
% \textit{Our code will be made publicly available when accepted.}

\section{Related work and Preliminaries}  
\label{sec:related_work}
% In this section, we provide background of graph neural networks (\S \ref{subsec:gnn}) and neural architecture search (\S \ref{subsec:nas}). 
%\eat{
%\subsection{Graph Neural Networks} 
% \label{subsec:gnn}
% \subsubsection{Background}
 
{\bf Graph Neural Networks} (GNNs)
% Graph Neural Networks (GNNs) 
are powerful representation learning techniques \citep{xu2019powerful} with many key applications \citep{hamilton2018inductive}. Early GNNs are motivated from the spectral perspective, such as Spectral GNN~\citep{bruna2014spectral} that applies the Laplacian operators directly. ChebNet~\citep{defferrard2017convolutional} approximates these operators using summation instead to avoid a high computational cost. GCN~\citep{kipf2017semisupervised} further simplifies ChebNet by using its first order, and reaches the balance between efficiency and effectiveness, revealing the message-passing mechanism of modern GNNs.
Concretely, recent GNNs aggregate node features from neighbors and stack multiple layers to capture long-range dependencies.
For instance, GraphSAGE~\citep{hamilton2018inductive} concatenates nodes features with mean/max/LSTM pooled neighbouring information. GAT~\citep{Petar2018graph} aggregates neighbor information using learnable attention weights. GIN~\citep{xu2019powerful} converts the aggregation as a learnable function based on the Weisfeiler-Lehman test instead of prefixed ones as other GNNs, aiming to maximize the power of GNNs. %\textcolor{red}{we could shorten this if out of space.}
% Due to its importance and complexity, we choose to demonstrate the use of NAC for GNNs architecture search in this work, while it is also suited for other neural network applications.

\eat{
% \subsubsection{GNNs Formulation}
{\bf GNNs Formulation:}
It is useful to provide a layout of modern GNNs. Let $\mathcal{G}=(\mathcal{V}, \mathcal{E})$ denote an undirected graph having node feature vectors $\boldsymbol{x}_{i}$ for all $i \in \mathcal{V}$, and its adjacency matrix is  ${A} \in\{0,1\}^{|\mathcal{V}| \times|\mathcal{V}|}$. 
% We impose an graph on these nodeS based on their node features, denoting as  $\mathcal{G}^F=(\mathcal{V}, \mathcal{E}^F)$, where the adjacency matrix is written as  ${A}^F \in\{0,1\}^{|\mathcal{V}| \times|\mathcal{V}|}$. 
% To address potential numerical issues, we often apply the renormalization trick that adding self-loops for all nodes, also called as augmented normalized adjacency matrix, 
% which is written as $\tilde{{D}}^{-1 / 2} \tilde{{A}} \tilde{{D}}^{-1 / 2}$ with $\tilde{{A}}={A}+{I}$ and $\tilde{{D}}={D}+{I}$. 
The corresponding normalized Laplacian matrix
is defined as $L=I-{{D}}^{-1 / 2} {{A}} {{D}}^{-1 / 2} = U \operatorname{diag}(\boldsymbol{\lambda}) U^{T}$.
}

GNNs consist of two major components, 
% i.e., aggregation and combination. 
% More specifically, 
where the {\em aggregation} step aggregates node features of target nodes' neighbors and the {\em combination} step passes previous 
aggregated features to networks to generate node embeddings. 
Mathematically, we can update node $v$'s embedding at the $l$-th layer by 
%$\mathbf{h}_{v}^{l}=\phi\left[ {W}^{l} \cdot \mathbf{A G G}_{\text {node }}\left(\left\{\mathbf{h}_{u}^{l-1}, \forall u \in {N}(v)\right\}\right)\right]$,
\begin{equation}\label{eq:GNN}
\bm{h}_{v}^{l}=\phi\left( \bm{W}^{l} \cdot o\left(\left\{\bm{h}_{v}^{l-1}, \forall u \in {N}(v)\right\}\right)\right),
\end{equation}
%$\mathbf{h}_{v}^{l}=\phi\left[ {W}^{l} \cdot o\left(\left\{\mathbf{h}_{u}^{l-1}, \forall u \in {N}(v)\right\}\right)\right]$,
where ${N}(v)$ denotes the neighbours of $v$. $\bm{W}^{l}$ denotes the trainable weight shared by all nodes at the $l$-layer.  $\phi$ is an activation function (with ReLU as default). 
{A main difference of different GNNs lies in the design of the the aggregation functions, $o(\cdot)$.  As GNNs have a large design space~\citep{you2020design}, this makes it challenging for the architecture search.
}
\eat{The difference among GNNs mostly lies in 
% the definition of 
aggregation functions, $o(\cdot)$, also known as operators.
% It is easy to see 
GNNs have a large design space \citep{you2020design}, making its architecture design challenging.}
% of GNN architectures a challenging task. 
% \textcolor{red}{It is easy to see GNNs have a large design space \citep{you2020design}, which makes the design of GNN architectures a challenging task.}
% See \citep{hamilton2018inductive} for more details.
% }

% \subsection{Neural Architecture Search} 
% \label{subsec:nas}
{\bf Neural Architecture Search} 
% \subsubsection{Background} 
% {\bf Background:}
~\citep{xie2017genetic} generally includes three major parts, i.e., search spaces, search strategies, and evaluation methods. 
% Notably, new NAS approaches are mainly focused on search strategies and evaluation methods, which are often mutually influenced/dependent.
% % In the early stages, the architecture search is implemented 
% Early NAS approaches mostly sample possible architectures, and then train on all to pick the best.
% Methods in this category often use reinforcement learning~\citep{DBLP:conf/iclr/ZophL17,baker2017designing}  or evolutionary algorithm~\citep{real2017largescale,so2019evolved}.
% Clearly, these methods suffer from high computational costs by excessive evaluation.
% because of the enumeration and individual evaluation. 
Recent advances adopt weight-sharing~\citep{xie2020weightsharing} to avoid (re-)training from scratch as in previous methods, where any architecture can share the weights directly for evaluation without re-training.
In particular, making the weights and architectural parameters differentiable~\citep{liu2019darts,luo2019neural} further improves the training efficiency.
Although efficient, weight-sharing NAS has genetic issues that are caused by its bi-level optimization, including over-fitting and inaccurate gradient estimation~\citep{xie2020weightsharing}, making these methods collapse unexpectedly. A variety of methods are proposed to address this from different perspectives, while mainly adding regularization towards weights, 
such as using $k$-shot~\citep{su2021kshot},  paths dropout~\citep{bender18a}, sampling paths~\citep{guo2020powering}, sparsity~\citep{zhang2018search}.
However, these cures are still within the bi-level optimization framework, which can at best alleviate the problem but not {directly tackle it.}  \eat{address it.}
In this work, we propose NAC, a linear (i.e., single-level) optimization paradigm, to address the issues in weight-sharing NAS.

Concretely, consider WS-NAS  defined by $\left \{\mathcal{A}, \mathcal{W}  \right \}$, where $\mathcal{A}$ and $\mathcal{W} $ denote the search space and weights of neural networks, respectively.
WS-NAS addresses two subproblems of these two space alternatively using a bi-level optimization framework~\citep{liu2019darts}, including weight optimization and architecture optimization as shown in the following.
% \begin{equation}~\label{darts}
%     \begin{aligned}
% \begin{cases}
% w=\underset{w}{\operatorname{argmin}} \mathcal{L}_{\text {train }}(\boldsymbol{\alpha}, w) \\ 
% \boldsymbol{\alpha}^{*} =\underset{\boldsymbol{\alpha}}{\operatorname{argmax}}\mathcal{M}_{va l}\left(w^{*}(\boldsymbol{\alpha}), \boldsymbol{\alpha}\right),
% \end{cases}
% \end{aligned}
% \end{equation}
%\textcolor{blue}
% \begin{equation}~\label{darts}
%     \begin{aligned}
% \begin{cases}
% \boldsymbol{\alpha}^{*} =\underset{\boldsymbol{\alpha}}{\operatorname{argmax}}~\mathcal{M}_{va l}\left(w^{*}(\boldsymbol{\alpha}), \boldsymbol{\alpha}\right)\\ 
% w^{*}(\boldsymbol{\alpha}) =\underset{w}{\operatorname{argmin}}~\mathcal{L}_{\text {train }}(\boldsymbol{\alpha}, w),
% \end{cases}
% \end{aligned}
% \end{equation}
\begin{equation}~\label{darts}
    \begin{aligned}
\boldsymbol{\alpha}^{*} &= \underset{\boldsymbol{\alpha}}{\operatorname{argmax}}~\mathcal{M}_{val}\left(w^{*}(\boldsymbol{\alpha}), \boldsymbol{\alpha}\right),\\
\quad \mbox{s.t.} \quad & \boldsymbol{w}^{*}(\boldsymbol{\alpha}) =\underset{\boldsymbol{w}}{\operatorname{argmin}}~\mathcal{L}_{\text {train }}(\boldsymbol{\alpha}, \boldsymbol{w}),
\end{aligned}
\end{equation}
where $\mathcal{L}$ denotes a loss function in the training set.  $\mathcal{M}$ denotes an evaluation metric, such as accuracy, in the validation set. 
Optimizing separately,  WS-NAS avoids the potential 
biased caused by imbalance dimensionality between $w$ and $\boldsymbol{\alpha}$, i.e., the neural weights and the architecture parameters.  DARTS~\citep{liu2019darts} further
converts the discrete search space into a differentiable one with the help of the softmax function that is writing as $ \frac{\exp \left(\alpha_{o}^{(i, j)}\right)}{\sum_{o^{\prime} \in \mathcal{O}} \exp \left(\alpha_{o^{\prime}}^{(i, j)}\right)}$,
% i.e. $\frac{\exp \left(\alpha_{o}^{(i, j)}\right)}{\sum_{o^{\prime} \in \mathcal{O}} \exp \left(\alpha_{o^{\prime}}^{(i, j)}\right)}$, 
% \begin{equation}
%     \begin{aligned}
%          \frac{\exp \left(\alpha_{o}^{(i, j)}\right)}{\sum_{o^{\prime} \in \mathcal{O}} \exp \left(\alpha_{o^{\prime}}^{(i, j)}\right)}, 
%     \end{aligned}
% \end{equation}
where 
$\mathcal{O}$ represents the search space consisting of all operators. 
%\textcolor{red}{$\boldsymbol{\alpha}$ is not explained}
% \begin{equation}
% \footnotesize
%     \begin{aligned}
%     \bar{o}^{i j}(x)=\sum_{o \in \mathcal{O}} \frac{\exp \left({\boldsymbol{\alpha}}_{o}^{i j}\right)}{\sum_{o^{\prime} \in \mathcal{O}} \exp \left({\boldsymbol{\alpha}}_{o^{\prime}}^{i j}\right)} o(x)
%     \end{aligned}
% \end{equation},
Thus, we can get the importance of each operator in the format of probability.
However, it is difficult to reach the optimal state for the objectives defined in \Cref{darts} simultaneously, because achieving optimality in deep neural networks is still an open question, and the gradient w.r.t. $\boldsymbol{\alpha}$ is intractable~\citep{xie2020weightsharing}.
In practice, researchers often apply early stopping to tackle it, which leads to unexpected model collapse~\citep{zela2020understanding} and initialization sensitivity~\citep{xie2020weightsharing}.
More related work can be found at \Cref{appendix:related_work}.
% and thus lead inferior performance dominated by operators like skip-connection
% because bi-level optimization is the strongly NP-hard problem~\citep{Hansen1992} and often involves uncertainties in the solution~\citep{Sinha2018}. 
\eat{
Another line of work~\citep{chen2021neural}~\citep{shu2021nasi}  uses neural tangent kernels (NTK) to search the network structure, called training-free NAS and focuses on  CNN architectures. 
These hold strong assumptions when analyzing due to the need for infinite width of networks, and thus far from reality.
In contrast, we do not have such an assumption by taking advantage of the built-in linearity in GNNs to get untrained GNNs to work.
}

{\bf NAS for Graph Neural Networks:}
% \textcolor{red}{
% With so many potential operators to use, GNNs' performance highly depends on careful design architecture. 
% Indeed, 
Many NAS methods have been proposed and adapted for GNNs, %including [], [], [], [].
% Recent work has studied how to wrap GNN into the NAS framework, 
such as GraphNAS~\citep{gao2020graph}, Auto-GNN~\citep{zhou2019autognn}, and SANE~\citep{zhao2021search},  and thus enable learning an adaptive GNN architecture with less human effort.
% GraphNAS\todo{add here }.
% \textcolor{blue}{ 
% GraphNAS enables automatic design of the graph neural architecture based on reinforcement learning framework. 
Methods like Auto-GNN and GraphNAS adopt reinforcement learning (RL) for searching, and thus suffer from the expensive computational cost.
Thanks to the weight-sharing policy, SANE~\citep{zhao2021search} avoids heavy retraining processes in RL methods, which leads to superior performance in general.
However, this type of work also faces  issues from the weight-sharing differential framework, such as bias optimization and instability due to bi-level optimization. In this work, our proposed method NAC focuses on the NAS for GNNs, and we use GraphNAS and SANE as our main baselines.

{\bf Sparse Coding}~\citep{zhang2015survey} is a well-studied topic, and its goal is to learn an informative representation as a linear combination from a collection of basis or atom, which is called a \textit{dictionary} collectively.
%, and a subset of atoms is often named as a \textit{subdictionary}.
%  also called \textit{atoms}
% The transformed GNNs by Corollary~\ref{coro:gnn_dict} leads to a well-known research called  dictionary learning (DL) and 
% sparse coding (SC)~\cite{aharon2006ksvd}.
% leads to an equivalence between GNNs and sparse coding.
The standard SC formulation is written as follows:% 
% \todo{introduce over-complete?}
% \small
\begin{equation}
\begin{aligned}
\footnotesize
    \min_{\bm{\mathcal{D}},\bm{\Gamma}}~~ \mathcal{L}(\bm{\mathcal{V}},\bm{\mathcal{D}},\bm{\Gamma})&=
    \lVert \bm{\mathcal{V}} - \bm{\mathcal{D}} \bm{\Gamma}\rVert_F^2,  \\
   \mbox{s.t.}~~ \forall i, ||\bm{\Gamma}_{\cdot i}||_0\leq \tau_0; 
   %  s.t. \forall i, \mathcal{S}(\bm{\Gamma}_{\cdot i})\leq \tau_0;
    &\forall j, \lVert \bm{\mathcal{D}}_{\cdot j}\rVert_2=1,
\end{aligned}
\label{eq:sparsecoding}
\end{equation}
where $\bm{\Gamma}$ is the sparse coefficient vector, $\bm{\mathcal{D}}$ is the dictionary, $\tau_0$ is a value to control the sparsity. The first constraint is to encourage a sparse representation on $\bm{\Gamma}$ and the second constraint is to normalize the atoms in the dictionary. \eat{The first regularization is to encourage a sparse representation of the target $\bm{\mathcal{V}}$, and the second is to normalize the atoms, a.k.a the atom normalization.}NAS aims to find a small portion of operators from a large set of operators, making it a sparse coding problem in nature.  We, therefore, reformulate the NAS-GNN problem as a sparse coding problem.

\section{Network Architecture Coding (NAC) for NAS-GNN} 
\label{sec:nac}
To find the best-performing architecture, we introduce a novel NAS-GNN  paradigm, Neural Architecture Coding (NAC), based on sparse Coding, which requires no update on the neural weights during training.
Technical motivation for NAC stems from the observation that an untrained GNN performs well, and training it might not provide extra benefit.
{In the remainder of this section, we provide a theoretical analysis of why no-update GNNs are preferred (\Cref{subsec:theory}) and propose a new NAS-GNN paradigm based on our theorems by leveraging untrained GNNs (\Cref{subsec:nac}).}

\subsection{Analysis of No-update Scheme in GNNs} 
\label{subsec:theory}
In the performance analysis of GNNs, the nonlinear activation function $\phi$ defined in \Cref{eq:GNN} is often skipped. Following~\citep{wu2019sgc}, we simplify a GNN as $\bm{F}=\bm{A}\,\dots\,\bm{A}(\bm{A}\bm{X}\bm{W}_1)\bm{W}_2\dots\,\bm{W}_L=\bm{A}^L\bm{X}\prod_{l=1}^{L}\bm{W}_l$,
where $\bm{A}$ is the adjacency matrix, $\bm{W}_l$ is the neural weight at the $l$-th layer from a total number of layers, $L$.  We then obtain the final output as $Out = \bm{F}\bm{W}_o$ by passing a linear layer $\bm{W}_o$ . 
At the $t$-th iteration, we denote the GNN weight at the $l$-th layer as $\bm{W}_l(t)$ and the theoretical output layer weight as $\tilde{\bm{W}}_{o}(t)$.

By initializing the network for all layers, i.e., $[\bm{W}_{1:L}, \bm{W}_o]$, we aim to train the network to get the optimal weights, denoted as $[\bm{W}_{1:L}^*,\bm{W}_o^*]$, where we assume the optimal weight can be obtained at $+\infty$-th iteration: $\bm{W}_l^*=\bm{W}_l(+\infty)$.
We use gradient descent as the default optimizer in GNNs.
\eat{By default, we perform gradient descent to attain the optimizer in GNNs.}
\eat{By default, we adopt gradient descent as the optimizer in GNNs.}

We now provide our first main theorem to investigate why an untrained GNN model can attain the same performance as the optimal one.

\begin{theorem}
\label{theo:theo1}
Assume $\bm{W}_l(0)$ is randomly initialized for all $l\in [1, L]$, if  $\prod_{l=1}^{L}\bm{W}_l(0)$ is full rank, there must exist a weight matrix for the output layer, i.e., $\tilde{\bm{W}_o}$, that makes the final output the same as the one from a well-trained network:  
\begin{equation}
\label{eq:theo}
    \bm{A}^L\bm{X}\prod_{l=1}^{L}\bm{W}_l^*\bm{W}_{o}^*=\bm{A}^L\bm{X}\prod_{l=1}^{L}\bm{W}_l(0)\tilde{\bm{W}_{o}}.
\end{equation}
\end{theorem}

{Please see \Cref{theo:no-update_scheme} for more details.} This theorem implies that the output layer alone can ensure the optimality of the network even when all previous layers have no updates. Next, we show that training in the standard setting of GNNs is guaranteed to reach the desired weight of the output layer, i.e. $\tilde{\bm{W}_{o}}$, under mild conditions. We begin by providing the theorem from~\citep{NEURIPS2018Suriya}.

\begin{theorem}[From~\citep{NEURIPS2018Suriya}]
Assume we have the following main conditions: 1) the loss function is either the exponential loss or the log-loss; 2) gradient descent is employed to seek the optimal solution; and 3) data is linearly separable.

When defining a linear neural network as $ \bm{Y}=\bm{X}\bm{W}_1\bm{W}_2.... .\bm{W}_L=\bm{X}\prod_{l=1}^{L}{\bm{W}_l} = \bm{X} \bm{\beta}$, we always have the same optimal weight regardless of the number of layers when using gradient descent,
\begin{equation}
    \begin{aligned}
    \footnotesize
    \scriptsize
    \bm{\beta}^* &= \argmin_{\bm{\beta}}{\Vert\bm{\beta}\Vert^2}, \quad
    \mbox{s.t.} \quad \bm{X}\bm{\beta} \odot \bm{s}\ge \bm{1},
    \end{aligned}
\end{equation}
where $\bm{s}$ is the ground-truth label vectors with elements as $1$ or $-1$; $\bm{1}$ denotes a vector of all ones; $\odot$ denotes the element-wise product.
\label{theo:theo2}
\end{theorem}
This theorem suggests that the optimal weights obtained by the gradient descent are the same regardless of how many layers are in the network, which is equivalent to finding a max-margin separating hyperplane.

Following this, we prove that $\tilde{\bm{W}_o}$ can be obtained based on the above conditions:
\begin{theorem}
~\label{eq:theo2}
Assume a GNN model has either the exponential loss or the log-loss, the desired weight  $\tilde{\bm{W}_o}$ is secured when updating with gradient descent and  initializing $\prod_{l=1}^{L}\bm{W}_l(0)$ as an orthogonal matrix. Mathematically, we have $\hat{\bm{W}_{o}}(+\infty)=\hat{\bm{W}_{o}}^*=\tilde{\bm{W}_{o}}$. 
\end{theorem} 

{Please see \Cref{theo:no-update_scheme} for more details.} In summary, our proposed theorems prove that a GNN with randomly initialized weights can make the final output as good as a well-trained network, where one needs to update the overall network with gradient descent and initialize networks with orthogonal weights. Apart from the optimality, another immediate benefit of this no-update scheme is that the computational cost will be significantly reduced. This is of particular importance when the budget is limited, as the main computational cost of NAS comes from this update.

\eat{In this work, we treat the entire NAS seach space as a general GNN model where each layer is a mixture of multiple operators.} 

{The above theorems show that one can approximate the optimal output by using the network without training, where we only need to update the final linear layer of the search space. This finding motivates us to learn architectures by taking advantage of randomly initialized networks. 
We want to emphasize that real-world scenarios can break the optimal conditions, such as the high complexity of data and the early stopping of the training. Even though, our strategy still provides near-optimal performance, if not the best, in experiments. 
}

\subsection{Architecture Searching via Sparse Coding}  
\label{subsec:nac}
{In this work, we treat the entire NAS search space of GNNs as a general GNN model where each layer is a mixture of multiple operators. Inspired by~\citet{balcilar2021bridge}, we propose a unified expression for each GNN layer in the search space in the following theorem, whose proof is given in \Cref{proof:fixed}.
\begin{theorem}
~\label{theo:gnn_layer_format}
Simplifying the complex non-linear functions, each GNN layer in the search space can be unified by $\sum_k\mathrm{P}_{l}^k(L)\bm{H}_{l}\bm{W}_{l}^{k}$, where $\mathrm{P}(\cdot)$ denotes the fixed polynomial function with subscripts indicating individual terms, and superscript $(l)$ marks the current layer.
\end{theorem}
This theorem suggests that all GNN layers are unified as $\sum_k\mathrm{P}_{l}^k(L)\bm{H}_{l}\bm{W}_{l}^{k}$, as the output of each GNN operator shares the same expression. Recall that the non-linear activation functions $\sigma(\cdot)$ are often neglected in the performance analysis, and even removed in some cases. Following this, we propose the following corollary and provide the details in \Cref{proof:fixed}.
\begin{corollary}
~\label{theo:search_space_format}
The search space of GNNs can be unified as $\sum_{k=0}^K\mathrm{P}^k(\bm{L})\bm{X}\bm{W}^k=\bm{D}\bm{W}$ when removing the activation function,  where $\bm{D}=\Vert\{\mathrm{P}^k(\bm{L})\bm{X}\}$ is the fixed base, and $\bm{W}=\Vert\{\bm{W}^k\}^T$ is the trainable parameters. Here, $\Vert$ stands for concatenating a set of matrices horizontally.
\end{corollary}
This corollary derives a unified mathematical expression of the search space of GNNs, implying the natural connection between NAS in GNNs and the Sparse Coding.}

Building on the above theorems, we present a novel NAS-GNN paradigm in which an untrained GNN model can also yield optimal outputs under mild conditions in the following. This optimality allows us to waive the effort to design complicated models for updating weights and makes us focus on architecture updates. At a high level, the ultimate goal of NAS is to assign the coefficients of unimportant operators as zeros and the important ones as non-zeros. Towards this, {searching the optimal architecture  is equivalent  to finding the optimal sparse combination of GNN operators}, i.e., $\boldsymbol{\alpha}$. Notably, one of the most successful algorithms for this purpose is sparse coding~\citep{OLSHAUSEN1997}, which has strong performance guarantees.

{\em What matters most when we reformulate the NAS problem as a sparse coding problem over a {dictionary} defined on GNN operators?}
According to the research in sparse coding~\citep{Aharon2006,Tropp2004,Candes2006}, its performance depends on the {quality} of the dictionary. In particular, the most desirable property is mutual coherence or orthogonality. High incoherence in the dictionary tends to avoid ambiguity in the solution and improve optimization stability, resulting in optimal or at least sub-optimal solutions. 
We prove a dictionary that is defined on deep neural networks has built-in orthogonality, as shown in the theorem below. % 
\begin{theorem}
Let the neural weights of each operator in a deep neural network be an atom and stack them column-wisely, we can guarantee the existence of an orthogonal dictionary.
\label{theo_linear}
\end{theorem}
Please see \Cref{theo_linear_proof} for more details. This theorem implies that the obtained weights are qualified to be a dictionary due to the built-in high dimensionality in deep neural networks.  

Taking inspiration from this, we propose a new NAS paradigm called \textit{neural architecture coding} as follows: 
\begin{equation}
\label{nac_obj}
\begin{aligned}
\footnotesize
 &\min_{\boldsymbol{\alpha}}\quad \mathcal{L}\left (\bm{y}, f(\mathbf{z}) \right) + \|\mathbf{z} - \mathbf{h}^{L}\|_F^2 + \rho\|\boldsymbol{\alpha}\|_{1},\\
  & \mbox{where}
  \begin{cases}
  \footnotesize
\mathbf{h}_{v}^{l}=\phi\left[ {\bm{W}}^{l} \cdot \bar{o}^{l}\left(\left\{\mathbf{h}_{u}^{l-1}, \forall u \in {N}(v)\right\}\right)\right] \\ 
\bar{o}^{l}(x) = \bm{o}^{l}\hat{\boldsymbol{\alpha}_{l}} =  \sum_{k=1}^{K} \frac{\alpha_{lk}}{|| \bm{\alpha_{l}}||_{2}}o^{lk}(x)\\
\bm{o}^{l} = [o^{l1}(x), o^{l2}(x), \dots, o^{lK}(x)]
\end{cases}
\end{aligned}
\end{equation}
where $\rho$ is the sparsity hyperparameter, $f(\cdot)$ is a linear function as the final output layer, e.g., MLP, $\mathbf{h}$ represents all nodes' embeddings.
%, and  $\mathcal{L}_{CE}$ defines the Cross-Entropy loss function. 
Here, we set the number of GNN operators to $K$ and the number of hidden layers to $L$.
$o^{lk}$ denotes the $k$-th operator at the $l$-th layer and its output is a vector; $\alpha_{lk}$ denotes the scalar coefficient for operator $k$ at $l$th-layer and thus $\boldsymbol{\alpha}_{l}\in \mathbb{R}^{k}$; $\bm{o}^{l}$ is the weight vector from all operators at the $l$-th layer.
$\mathbf{h}_{v}^{l}$ denotes the embedding of the node $v$ at the {$l$-th layer}, where its neighborhood aggregation process is performed as the weighted summation of the outputs from all operators as $\bar{o}^{l}(x)$. { Here, $\bf y$ is the node label.}

Notice that a standard sparse coding model uses an l2-norm-based objective for optimization, while a cross-entropy objective is preferred as we focus on a classification task. To connect the gap, we prove that optimizing the Cross-Entropy loss function in GNNs satisfies a sufficient condition for the l2-norm-based objective in sparse coding, as shown in the following theorem.
% To optimize our model, we employ the cross-entropy function as the loss function and take gradient descent for optimization.
% We begin by stating our theorem of optimizing the $l2$-norm objective below:
\begin{theorem}
\label{theo:ce_for_l2}
The gradient from Cross-Entropy loss function is feasible to optimize an $l2$-norm objective in a gradient descent manner, such that the initial $z^0$ converges geometrically to the optimal $z^*$ with a tiny error, e.g., 
\begin{equation}
    \lVert z^{s+1}-z^* \rVert \leq (1-2\alpha \eta)^s\lVert z^0-z^* \rVert + 2\gamma/\alpha,
\end{equation}
where $s$ signifies the current iteration, and $\gamma, \alpha,\eta$ define the numerical bound for the error.
\end{theorem}
Please see \Cref{proof:equal} for more details. 
Therefore, we reformulate the objective in \Cref{nac_obj} as follows,
\begin{equation}
\label{nac_obj_ce}
    \begin{aligned}
    \min_{\boldsymbol{\alpha}}\quad \mathcal{L}_{CE}\left (\bm{y}, f(\mathbf{z}) \right) + \|\mathbf{z} - \mathbf{h}^{L}\|_F^2 + \rho\|\boldsymbol{\alpha}\|_{1},
    \end{aligned}
\end{equation}
where $\mathcal{L}_{CE}$ defines the Cross-Entropy loss function.

\begin{algorithm}[!ht]
\caption{The NAC algorithm }
~\label{alg_nac1}
\begin{algorithmic}[1]
\footnotesize
% \scriptsize
\REQUIRE The search space $\mathcal{A}$;
\ENSURE The architecture $\boldsymbol{\alpha}$\\
Randomly initializing $\bm{W}^l$, for $l=1, \ldots, L$;
set $\boldsymbol{\alpha}=\bf{1}$;
\WHILE{$t = 1, \dots, T$}
    \STATE Performing the feature aggregation at each layer as $\bar{o}^{l}(x) = \bm{o}^{l}\hat{\boldsymbol{\alpha}_{l}} = \sum_{k=1}^{K} \frac{\alpha_{lk}}{|| \bm{\alpha_{l}}||_{2}}o^{lk}(x)$;
    \STATE Computing  $\mathbf{h}_{v}^{l}=\phi\left[ {\bm{W}}^{l} \cdot \bar{o}^{l}\left(\left\{\mathbf{h}_{u}^{l-1}, \forall u \in {N}(v)\right\}\right)\right]$;
    \STATE Optimizing $\boldsymbol{\alpha}$ based on the objective function in \Cref{nac_obj_ce} w.r.t. $\boldsymbol{\alpha}$ under the fixed dictionary $\bm{o}$ ;% as in Eq.7; 
    \STATE Updating $\bm{W}_o$ based on the objective function in \Cref{nac_obj_ce} w.r.t.  $\bm{W}_o$ under fixed $\boldsymbol{\alpha}$;
\ENDWHILE
% \State Obtain the final architecture $\{\boldsymbol{\alpha}^{*}\}$ from the trained $\boldsymbol{\alpha}$;
\STATE Obtain the final architecture $\{\boldsymbol{\alpha}^{*}\}$ from the trained $\boldsymbol{\alpha}$ via an argmax operation at each layer;
\end{algorithmic}
\end{algorithm}

It is important to note that, rather than computing the optimal neural weight for $o^{lk}$ like previous NAS methods explicitly,
we use randomly initialized weights instead and focus on optimizing the combination coefficient $\boldsymbol{\alpha}$, i.e. architecture parameters. The desired optimal embedding $\mathbf{z}$ is then used for downstream applications, such as node classification.
In practice, we often let $\mathbf{z} = \mathbf{h}^{L}$ for computational convenience, which makes the loss of the second term to be zero.

The sparsity regularization $||\boldsymbol{\alpha}||_{1}$ allows us to rank the importance of the operator directly, alleviating inaccurate gradient estimation such as softmax~\citep{garg2021alternate}. We want to emphasize that this regularization does not break the requirement of using gradient descent. Since the architecture, i.e., $\boldsymbol{\alpha}$,  is fixed when updating the output layer weights, we can use gradient descent for updating as required, so the theorems in \Cref{subsec:theory} hold.

\subsection{Implementation Details and Complexity Analysis}  
The overall algorithm is presented at \Cref{alg_nac1}. Computing NAC consists of two major parts: the forward pass and the backward pass. Given the search space, the computation of the forward is then fixed and regarded as a constant. Therefore, the computational complexity mainly focuses on the backward pass in the NAC algorithm.

In summary, the algorithmic complexity of NAC  is  $O(T*| \boldsymbol{\alpha}|)$, where $| \boldsymbol{\alpha}|$ is the size of $ \boldsymbol{\alpha}$. 
% Note the dimension of $w$ is often much larger than $\boldsymbol{\alpha}$.
{Please refer to \Cref{subsec:impl} for more details about the implementation.}

\begin{table*}[!th]
\centering
\caption{Experimental results on the compared methods: our NAC attains superior performance in both accuracy (\%) and efficiency (in minutes).
\eat{(RS: random search; BO: Bayesian Optimization; RL: reinforcement learning; WS: weight sharing.)}
}

\scalebox{0.9}{
\begin{tabular}{l|cccccccc}
\toprule
  %~&
  \multirow{2}{*}{} &
  \multicolumn{2}{c}{CiteSeer} &
  \multicolumn{2}{c}{Cora} &
  \multicolumn{2}{c}{PubMed} &
  \multicolumn{2}{c}{ Computers}
  \\
  ~ & %~& 
  {Accuracy} & {Time} & {Accuracy} & {Time} & {Accuracy} & {Time} & {Accuracy} & {Time}\\
  \midrule
    RS %& Random Search 
    & 70.12${_{\pm2.36}}$ & 14.4 & 71.26${_{\pm4.68}}$ & 30.6 & 86.75${_{\pm0.82}}$ & 187.8 & 77.84${_{\pm1.35}}$ & 8.75\\
    BO %& Bayesian      
    & 70.95${_{\pm1.62}}$ & 18   & 68.59${_{\pm6.66}}$ & 31.2 & 87.42${_{\pm0.68}}$ & 189.6 & 77.46${_{\pm2.02}}$ & 17.65\\
    %\multirow{2}{*}{RL} & 
    GraphNAS     & 68.69${_{\pm1.30}}$ & 253.8 & 71.26${_{\pm4.90}}$ & 245.4 & 86.07${_{\pm0.51}}$ & 1363.8 & 73.97${_{\pm1.79}}$ & 86.37 \\
                    %    & 
    GraphNAS-WS  & 65.35${_{\pm5.13}}$ & 80.4  & 72.14${_{\pm2.59}}$ & 161.4 & 85.71${_{\pm1.05}}$ & 965.4  & 72.99${_{\pm3.44}}$ & 42.47 \\
    %WS-NAS & 
    SANE & 71.84${_{\pm1.33}}$ & 4.2  & 84.58${_{\pm0.53}}$ & 10.2 & 87.55${_{\pm0.78}}$ & 107.4 & 90.70${_{\pm0.89}}$ & 0.72\\ \midrule
    %\multirow{1}{*}{NAC }   & 
    NAC      & \textbf{74.62${_{\pm0.38}}$} &\textbf{1.2} & \textbf{87.41${_{\pm0.92}}$} & \textbf{1.2} & {88.04${_{\pm1.06}}$} &\textbf{9.0} & \textbf{91.64}${_{\pm0.14}}$& \textbf{0.23} \\
                    %    ~   & 
                        NAC-updating       & 74.17${_{\pm1.18}}$ & 4.2 & 86.62${_{\pm1.14}}$ & 3.6 & \textbf{88.10}${_{\pm0.86}}$ & 25.8 & 90.89${_{\pm1.10}}$ & 0.70 \\
     \bottomrule
\end{tabular}
}
\label{Table_main_comparison111}
\end{table*}

\section{Experiments}\label{sec:4}
\eat{We evaluate the performance and properties of the proposed NAC to answer the following research questions:}
We conduct experiments to address the following issues:
\textbf{(1)} \textit{How does NAC perform in comparison to the leading baselines?}
\textbf{(2)} \textit{How does the no-update scheme affect other methods?}
\textbf{(3)} \textit{Is NAC robust?}
\Cref{sec:main_results}-\Cref{subsec:convergence} answer the above questions accordingly.
\eat{
For (1), we compare NAC with five strong baselines. \todo{several is replaced to the number}
For (2), we perform the no-update scheme on the best-performing baseline, SANE, and see if this would improve its performance.
For (3), we investigate NAC on different initializations, including random seeds, sparsity parameters, and linear output layer training. \todo{different is replaced to three?}
}

%%%%%%%%%%%%%%%%%%%%%%%%%%%%%%%%%%%%%%%%%%%%%%%%%%%%

\subsection{Experiment Setup} \label{subsec:setup}
\noindent\textbf{Datasets.}
We performed experiments on transductive learning for node classification~\citep{zhao2021search}.
% In this task, only a subset of nodes in one graph is used for training, while other nodes are used for validation and test, respectively. \todo{in one graph is necessary? } 
For this setting, we use four benchmark datasets: CiteSeer~\citep{Citeseer}, Cora~\citep{Cora}, PubMed~\citep{PubMed}, and Computers~\citep{mcauley2015image}. Also, we follow the data partition setting (training/validation/testing) as in ~\citep{zhao2021search}. {For more details about transductive learning task and introduction of each dataset, please refer to \Cref{subsec:dataset} for more details.}

{We also conduct additional experiments on inductive task like PPI \Cref{appendix:exp_ognb} dataset\cite{hamilton2018inductive} and the Open Graph Benchmark datasets\citep{hu2020open} in \Cref{appendix:exp_ognb}.}

\noindent\textbf{Environment}. 
{ For datasets of Citeseer, Cora, and PubMed, we run experiments on a CPU server with a $48$-core Intel Xeon Platinum 8260L CPU using  PyTorch~\citep{NEURIPS2019_9015} (version 1.10.2+cpu). Due to the large size of Amazon Computers, we experiment on a GPU server with four NVIDIA 3090 GPUs (24G) with PyTorch (version 1.10.2+gpu). 
% \eat{Speed-related experiments on Citeseer, Cora, and PubMed were measured using PyTorch (version 1.4+cpu) on a CPU server with a ten-core Intel Xeon Platinum 8255C CPU, 40G RAM, and 500G DRAM.} \eat{Speed-related experiments on Amazon Computers were measured using PyTorch (version 1.10.2+gpu) on a GPU server with four NVIDIA 3090 GPUs (24G).} 
All operators used in the experiments are from the built-in functions of PyG  (version 2.0.2)~\citep{Fey/Lenssen/2019}.}

\textbf{Methods.}
% In order to have a comprehensive comparison, 
{Following~\citep{zhao2021search}, we compare our NAC with the following strong baselines: 1) Random Search (RS)~\citep{Bengio2012hpo} and 2) Bayesian Optimization (BO)~\citep{Jones1998}, 3) GraphNAS~\citep{gao2020graph}: a popular reinforcement learning-based (RL) method, 4) GraphNAS-WS~\citep{zhao2021search}: a variant of GraphNAS with weight sharing, and 5) SANE~\citep{zhao2021search}.  Among all baselines, the first two are hyperparameter optimization (HPO) methods~\citep{yu2020hpo}.  GraphNAS and GraphNAS-WS are two popular methods following the weight sharing paradigm while SANE is the most recent work on automated GNNs using the weight-sharing differential paradigm, which is the closest one to our work.  We also evaluate NAC-updating, i.e., NAC with weight updates, to compare it with the proposed one with no-update scheme.  
}
\eat{We include diverse NAS approaches with different search strategies (i.e., reinforcement learning methods, SOTA WS-NAS methods, and widely used methods in hyperparameter tuning) in our experiment. \todo{need rephrase}
}
\eat{Specifically, the baselines include a popular {reinforcement learning-based} method,  GraphNAS~\citep{gao2020graph}, and its variation GraphNAS-WS~\citep{zhao2021search} that uses weight sharing for evaluation. Additionally, our baselines also include SANE~\citep{zhao2021search}---the most recent work on automated GNNs using {the weight-sharing differential paradigm}, which is the closest to our work.

Due to the similarity between NAS and hyperparameter optimization (HPO)~\citep{yu2020hpo}, researchers often include some {HPO methods} as baselines in NAS as well, such as Random Search (RS)~\citep{Bengio2012hpo} and Bayesian optimization methods~\citep{Jones1998}, so do we.
In our experiments, we also follow the setting in~\citep{zhao2021search}, using RS and Bayesian for comparison. We also include a variation of  NAC by updating its weigths, called NAC-updating, to demonstrate the superiority of the non-updating scheme.
Note that all results of baselines are reported under the best effort of parameter tuning. 
}

\textbf{Search space.}
We select a list of node aggregators (i.e., operators) in GNNs to form the search space.
More specifically, they include the following seven aggregators: multi-layer perceptrons (MLP), GCN~\citep{kipf2017semisupervised}, GAT~\citep{Petar2018graph}, GIN~\citep{xu2019powerful}, Geniepath~\citep{liu2018geniepath}, Graphsage~\citep{hamilton2018inductive}, and ChebNet~\citep{defferrard2017convolutional}.
Note that MLP is a special aggregator with fixed frequency when viewing GNNs from a spectrum perspective~\citep{balcilar2020analyzing}.
Besides, Graphsage and GAT contain different aggregation functions, i.e., maximum and LSTM.   We apply all these variations as operators in the experiments. 
Note that we do not include the layer aggregators like skip-connection in our search space.
The underlying rationale is that recent research~\citep{zela2020understanding,chen2021stabilizing,liang2020darts} finds that NAS models converge to skip-connection much faster than other operators due to its parameter-free nature.  This introduces biased searching and yields poor generalization. \textit{Our experiments ensure that the search space is the same for all baselines.}

We do not include layer-wise aggregators like JKNets~\citep{pmlr-v80-xu18c} in the main experiment. As we show in \Cref{subsec:theory}, our theory is built upon the classical GNN framework, which equals matrix multiplication. JKNets cannot be represented in this framework, which contradicts our assumptions.

\textbf{Evaluation Metrics.}
The proposed approaches aim to boost the ability to learn efficacious architecture for downstream tasks.  The classification accuracy and the runtime (i.e., wall-clock time) are two widely used metrics in NAS research to measure the model performance and efficiency, respectively~\citep{zhang2018search}. The reported results are the average of the best results in 4 runs from different random seeds.

\eat{Thus, we report the classification accuracy and the runtime (i.e., wall-clock time), which are widely used in NAS research~\citep{zhang2018search}.
% More specifically, we use classification accuracy as the evaluation metric. 
Note that all reported results are the average of the best results from 
% from the  top-$1$ ranked architectures 
different random seeds.}

We set the runtime of NAC, SANE, and GraphNAS as the time they take to obtain the final architectures with $100$ epochs\footnote{The running time for the Computers dataset is measured on a single GPU, which is different from all other three datasets.}. The runtime of RS and Bayesian is fixed to the time of $100$ times trial-and-error processes.

\subsection{Comparison Results}%with Leading NAS Methods} 
\label{sec:main_results}
Results in \Cref{Table_main_comparison111} and \Cref{fig:overview} show that NAC attains superior performance than all baselines regarding both \textit{accuracy} and \textit{efficiency}. More specifically, we observe the following:
\begin{compactitem}[-]
\item In terms of model performance, our NAC beats all baselines and attains up to $2.83\%$%  $0.55\%$-$2.83\%$
improvement over the best baseline, i.e., SANE, while attaining %$0.82\%$-$18.8\%$ 
up to $18.8\%$ 
improvement over the Bayesian method, the best HPO method.
Thanks to our non-updating scheme, it prompts the outputs near the optimal to make the selection of the best-performing architecture reliable,
as opposed to biased optimization in others.
% While existing methods often struggle to train the network, suffering either biased optimization and high computational costs.
% and resulting in inferior performance to NAC.
We notice that in one case, NAC is slightly worse than NAC-updating, about $0.06\%$.
This case reminds us that the optimal condition can get compromised due to the high complexity of the data.
Still, NAC holds for near-best performance, demonstrating its robustness.
Empirical results corroborate our theories in \Cref{subsec:theory} that non-updating scheme is preferred.
The underlying assumption of RL-based and BO-based methods is to obtain an accurate estimation of the distribution of the search space. 
However, these methods rely on random sampling to perform the estimation,
%and the sampled data can be very different with the distribution of the search space,
thus have no guarantee of the search quality due to the high complexity in the search space.
For instance, RL-based methods are sometimes even worse than RS-based methods, pure randomly sampled, implying the unreliability of such estimation under a limited budget.
% RL-based and BO-based methods suffer from the inaccurate estimation due to random sampling.
WS-based methods couple the architecture search and the weight training, so they struggle to train the network to be optimal and obtain a biased network due to oscillations between the two components.
\item In terms of model efficiency, our NAC achieves superior performance, around $10\times$ faster than SANE and {up to $200\times$ time faster than} GraphNAS.  
Our non-updating scheme requires nearly no weight update, thus giving NAC an incomparable advantage in reducing the computation cost.
All previous NAS methods need to update neural weights, for example, RL-based costs the most due to retraining the network from scratch at each time step, and WS-based costs the least by reusing neural weights.
\end{compactitem}

\subsection{No-update Scheme at Work}
\label{sec:non-updating-exp}
To further validate our no-update scheme, we evaluate its effect on other weight-sharing methods. Since SANE attains the best performance among all baselines, we test the no-update scheme on SANE as $\text{SANE}^{*}$, i.e., SANE with fixed neural weights.  Results in \Cref{table: sane-updating} show that $\text{SANE}^{*}$ outperforms the one with updates.  This result implies that we can improve the performance of NAS-GNN methods by simply fixing the weights with random initialization.  This yields a much lower computational cost in the training.

\begin{figure*}[!thbp]
    \begin{minipage}[!thbp]{.55\linewidth}
        \centering
        \captionof{table}{ \footnotesize Performance comparison between SANE, NAC$^{+}$ and NAC.}
        \label{table:linearupdates}
        %{{{
        \resizebox{.78\linewidth}{!}
        {
            \begin{tabular}{c|cccccc}
                \toprule
                \multirow{1}{*}{} &
                \multicolumn{1}{c}{CiteSeer} &
                \multicolumn{1}{c}{Cora} &
                \multicolumn{1}{c}{PubMed} &
                \multicolumn{1}{c}{Computers}
                \\
                ~& {Acc(\%)} & {Acc(\%)}  & {Acc(\%)}  & {Acc(\%)}\\
                \midrule
                \multirow{1}{*}{SANE}  & 71.84$_{\pm 1.33}$ & 84.58$_{\pm 0.53}$ & 87.55$_{\pm 0.78}$ & 90.70$_{\pm 0.89}$\\
                \midrule 
                \multirow{1}{*}{NAC$^{+}$} & 71.76$_{\pm 2.08}$ & 87.27$_{\pm 1.20}$ & 87.59$_{\pm 0.22}$ & 91.07$_{\pm 0.71}$\\
                \midrule
                \multirow{1}{*}{NAC} & \textbf{74.62}$_{\pm 0.38}$ &\textbf{87.41}$_{\pm 0.92}$ & \textbf{88.04}$_{\pm 1.06}$ & \textbf{91.64}$_{\pm 0.14}$\\
                \bottomrule
            \end{tabular}
        }
        %}}}
        \captionof{table}{ \footnotesize Comparison between SANE and $\text{SANE}^{*}$ (w/o.~weight updates).}
        \label{table: sane-updating}
        %{{{
        \resizebox{.78\linewidth}{!}
        {
            \begin{tabular}{c|ccccc}
                \toprule
                \multirow{1}{*}{} &
                \multicolumn{1}{c}{CiteSeer} &
                \multicolumn{1}{c}{Cora} &
                \multicolumn{1}{c}{Pubmed} &
                \multicolumn{1}{c}{Computers}
                \\
                ~ & {Acc(\%)} & {Acc(\%)} & {Acc(\%)} & {Acc(\%)}\\
                \midrule
                SANE  & 71.84${_{\pm1.33}}$ & 84.58${_{\pm0.53}}$ & 87.55${_{\pm0.78}}$ & 90.70${_{\pm0.89}}$\\
                SANE$^{*}$  & \textbf{71.95}${_{\pm1.32}}$ & \textbf{85.46}${_{\pm0.76}}$ & \textbf{88.12}${_{\pm0.35}}$ & \textbf{90.86}${_{\pm0.80}}$\\
                \bottomrule
            \end{tabular}
        }
        %}}}
        \captionof{table}{ \footnotesize Comparison of NAC with different initializations.}
        \label{table:initial}
        %{{{
        \resizebox{.918\linewidth}{!}
        {
            \begin{tabular}{c|cccccc}
                \toprule
                \multirow{1}{*}{} &
                \multicolumn{1}{c}{CiteSeer} &
                \multicolumn{1}{c}{Cora} &
                \multicolumn{1}{c}{PubMed} &
                \multicolumn{1}{c}{Computers}
                \\
                ~& {Acc(\%)} & {Acc(\%)}  & {Acc(\%)} & {Acc(\%)}\\
                \midrule
                \multirow{1}{*}{Kaiming Normal}  & 71.12$_{\pm 2.45}$ & 86.85$_{\pm 0.78}$ & 87.52$_{\pm 0.72}$ & 91.05$_{\pm 0.80}$\\
                \midrule 
                \multirow{1}{*}{Kaiming Uniform} & 72.06$_{\pm 2.24}$ & 86.99$_{\pm 1.03}$ & 87.68$_{\pm 0.97}$ & 88.52$_{\pm 3.01}$ \\
                \midrule
                \multirow{1}{*}{Orthogonal} & \textbf{74.62}$_{\pm 0.38}$ &\textbf{87.41}$_{\pm 0.92}$ & \textbf{88.04}$_{\pm 1.06}$ & \textbf{91.64}$_{\pm 0.14}$\\
                \bottomrule
            \end{tabular}
        }
        %}}}
    \end{minipage}
    \begin{minipage}[!tbhp]{.38\linewidth}
        \centering
        \includegraphics[width=.82\linewidth]{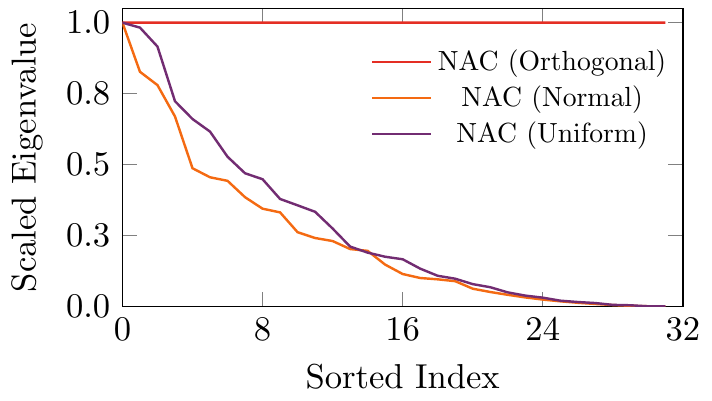}
        \caption{\footnotesize Visualization of the sorted eigenvalues of weights of NAC with different initialization. 
        }
        \label{fig: visal_svd}
        \includegraphics[width=.82\linewidth]{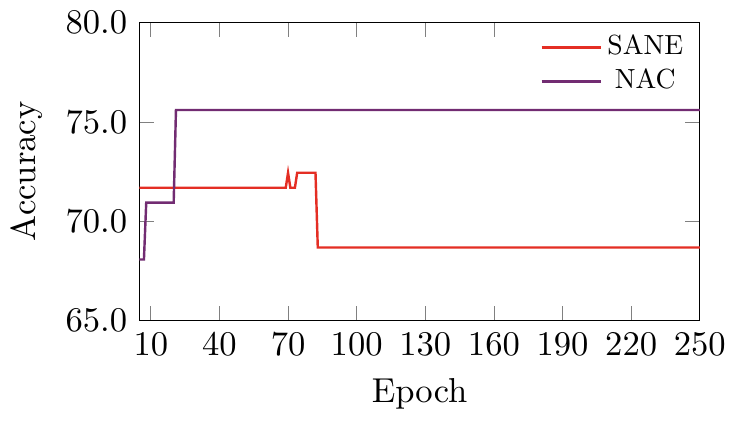}
        \caption{
            \footnotesize
            Convergence for SANE and NAC in terms of accuracy. NAC converges much faster than SNAE in only around $20$ epochs.  }
        \label{fig_convergence}
    \end{minipage}
\end{figure*}

%\begin{figure*}[!htbp]
%\begin{minipage}[!htbp]{.60\linewidth}
%    \centering
%    \scriptsize
%\end{minipage}
%% \hspace{0.1in}
%\begin{minipage}[!hbtp]{.368\linewidth}
%    \centering
%% 	\includegraphics[scale=0.616]{pics/NAC_Convergence.pdf}
%\end{minipage}
%\end{figure*}

\subsection{Analysis of Convergence}  
\label{subsec:convergence}
A notable benefit of the NAC framework is its guaranteed convergence from the sparse coding perspective.  To verify this, \Cref{fig_convergence} offers a convergence comparison between NAC and SANE from the Pubmed dataset by showing the accuracy on the retrained network acquired at each training epoch.
{NAC shows substantially faster convergence than SANE, which take around $20$ and $80$ epochs, respectively.
Due to the BLO strategy, SANE suffers strong oscillation when optimizing two variables, i.e., the architecture and neural weights, which consumes more epochs to converge. }

\subsection{Ablation Studies} 
\label{sec:ablation_studies}
{\bf Initialization:}
{Recall that in \Cref{subsec:theory}, a prerequisite is the orthogonality on the neural weights matrix.  This makes the initialization critical to the final performance.
To investigate the effect of initialization, we try three types of initialization methods, including normal, uniform, and orthogonal.  \Cref{table:initial} shows that NAC with orthogonal initialization outperforms other two cases of initialization.  The results again confirm our theory.

To get a better understanding of the reason, we visualize the eigenvalues of $\prod_{l=1}^L \bm{W}_l$ on different initialization cases in \Cref{fig: visal_svd}.  The results show that the eigenvalues of NACs with normal and uniform initialization decay accordingly.  This makes the input feature tends to project to the subspace spanned by the top-ranked eigenvectors, which results in feature distortion and leads to performance degeneration. 
In contrast, when NAC with orthogonal initialization, the eigenvalues are distributed uniformly, which ensures the spanned space does not bias toward any particular eigenvectors. 

In summary, NAC with orthogonal initialization attains the best performance and confirms our theory.  In addition, results from non-orthogonal initializations are still better than other baselines, demonstrating the robustness of NAC.

}

{\bf The linear output layer:}
Recall that for a given linear output layer, there exists the optimal weight, i.e. $\tilde{W_{o}}$,  to secure the optimal output according to \Cref{theo:theo1}.
Nevertheless, obtaining optimal weights in a deep learning algorithm is an open question due to its high dimensionality.
Often, in practice, we train the network with early stopping, especially in NAS methods.
Therefore,  obtaining the optimal weights for the output layer is not possible.
\textit{Is training the output layer necessary? } 
This experiment attempts to answer this by comparing NAC with and without training the output layer.
Note that we fix all other layers with random weights. 
By default, NAC is the one without training any layers in the whole paper, including the output layer.
On the contrary, $\text{NAC}^{+}$ is the one with training the output layer. Here, we set the number of epochs as $100$.
% Here, we name the NAC by training a linear layer with $100$ epochs as $NAC^{+}$.

In \Cref{table:linearupdates}, we can find that the performance of NAC$^{+}$ is in the middle between SANE and NAC. On the one hand, training the linear layer may not be worthwhile because it is impossible to get the desired optimal weights. On the other hand, results with and without training the output layer are always better than SANE, suggesting the superiority of the no-update scheme in general. To further demonstrate this, we implement multiple experiments by training the last layer with different epochs as shown in \Cref{subsubsec:traininglinear}.

\eat{
To further demonstrate this, we implement multiple experiments by training the last layer with different epochs.
%, as shown in \Cref{fig:random_linear_updating}.
\eat{as shown in \Cref{subsubsec:traininglinear}.}
{We notice that the model with untrained weights outperforms the one with trained weights that involves multiple epochs on the final linear layer, as shown in \Cref{fig:random_linear_updating}.
This inspires us to omit the training of the final linear layer in the main experiment, aligning with our proposed theorems in \Cref{subsec:theory}.
% We noticed that most of the time, the untrained weights in the initial state can often already exceed the accuracy that can be obtained from the weights after multiple epochs of training the final linear layer, as shown in \Cref{fig:random_linear_updating}.  
% Therefore, we further omit the training of the final linear layer. It is important to note that this approximation is based on our proposed theorems in which most of the intermediate layers do not require training.
}

\begin{figure}[!htbp]
    \centering
    \includegraphics[width=.35\textwidth]{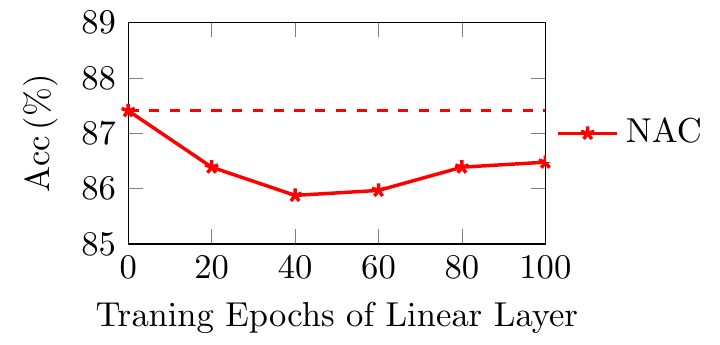}
    \caption{
        The effect of training final linear layer of NAC, where x-axis denotes the training epochs for the final linear layer and y-axis denotes the corresponding accuracy of acquired architecture $\bm{\alpha}$. % using the corresponding weights.
    }
    \label{fig:random_linear_updating}
\end{figure}
}

{\bf Sparsity effect: } {
%We also conduct experiments to test how sparsity and random seeds affect the search quality of NAC.
Our model uses the hyperparameter $\rho$ to control the sparsity of the architecture parameter $\bm{\alpha}$.
We investigate the effect of this hyperparameter by varying its value in a considerably large range, such as $\left [ 0.001,10 \right ]$. We find that the results vary little in a considerably wide range, this indicates NAC is insensitive to the sparsity hyperparameter.}
\eat{We also conduct experiments to test how sparsity affect the search quality of NAC.
The hyperparameter $\rho$ is to control the sparsity of the architecture parameter $\bm{\alpha}$.  We test the effect by varying its value in a considerably large range, i.e., $\left[0.001, 10\right]$. The results show that the performance varies little w.r.t. $\rho$.  This indicates that NAC is insensitive to the sparsity hyperparameter. Please refer to \Cref{subsubsec:theeffectofsparsity} for more details.
}

{\bf Random seeds: }
Random seeds often play an importance role in NAS methods as they also affect the initialization significantly. By default, researchers report the average or the best result from different random seeds, this often lead to poor reproducibility. We now present the effect of random seeds on this topic. In particular, we implement the experiment on multiple random seeds and observe a stable performance of NAC. The experimental results show that NAC is robust to both sparse hyperparameter selection and random seeds. Please refer to \Cref{subsubsec:randomSeeds} for more details.

%\todo{Need rewrite, what is BLO?}

\section{Conclusion }  
We present the first linear complexity NAS algorithm for GNNs, namely NAC.  Our NAC is solved by sparse coding and holds no-update scheme for model weights in GNNs because GNNs embed the built-in model linearity and orthogonality for the model weights.  Extensive experiments demonstrate that our NAC can achieve higher accuracy (up to 18.8\%) and much faster convergence (up to 200$\times$) than SOTA NAS-GNNs baselines.  

Several promising directions can be considered. We can further explore more deep neural networks that satisfy the mild condition of our NAC to extent its usability.  We can investigate more on the efficiency of different subgradient methods for solving the sparse coding objective. We also intend to investigate how to jointly learn the search space and architectural representation to further enhance the expressive ability of searched architectures.

\newpage

\bibliographystyle{icml}
\bibliography{ref/Top-sim,ref/all,ref/gcf_reference}

\begin{thebibliography}{66}
\providecommand{\natexlab}[1]{#1}
\providecommand{\url}[1]{\texttt{#1}}
\expandafter\ifx\csname urlstyle\endcsname\relax
  \providecommand{\doi}[1]{doi: #1}\else
  \providecommand{\doi}{doi: \begingroup \urlstyle{rm}\Url}\fi

\bibitem[Aharon et~al.(2006)Aharon, Elad, and Bruckstein]{Aharon2006}
Aharon, M., Elad, M., and Bruckstein, A.
\newblock K-svd: An algorithm for designing overcomplete dictionaries for
  sparse representation.
\newblock \emph{IEEE TSP}, 54\penalty0 (11):\penalty0 4311--4322, 2006.

\bibitem[Arora et~al.(2015)Arora, Ge, Ma, and Moitra]{arora2015simple}
Arora, S., Ge, R., Ma, T., and Moitra, A.
\newblock Simple, efficient, and neural algorithms for sparse coding.
\newblock In \emph{Proc.~COLT}, pp.\  113--149, 2015.

\bibitem[Balcilar et~al.(2020)Balcilar, Renton, H{\'e}roux, Ga{\"u}z{\`e}re,
  Adam, and Honeine]{balcilar2020analyzing}
Balcilar, M., Renton, G., H{\'e}roux, P., Ga{\"u}z{\`e}re, B., Adam, S., and
  Honeine, P.
\newblock Analyzing the expressive power of graph neural networks in a spectral
  perspective.
\newblock In \emph{Proc.~ICLR}, 2020.

\bibitem[Balcilar et~al.(2021)Balcilar, Renton, H{\'{e}}roux,
  Ga{\"{u}}z{\`{e}}re, Adam, and Honeine]{balcilar2021bridge}
Balcilar, M., Renton, G., H{\'{e}}roux, P., Ga{\"{u}}z{\`{e}}re, B., Adam, S.,
  and Honeine, P.
\newblock Analyzing the expressive power of graph neural networks in a spectral
  perspective.
\newblock In \emph{Proc.~ICLR}, 2021.

\bibitem[Bender et~al.(2018)Bender, Kindermans, Zoph, Vasudevan, and
  Le]{bender18a}
Bender, G., Kindermans, P.-J., Zoph, B., Vasudevan, V., and Le, Q.
\newblock Understanding and simplifying one-shot architecture search.
\newblock In \emph{Proc.~ICML}, 2018.

\bibitem[Bergstra \& Bengio(2012)Bergstra and Bengio]{Bengio2012hpo}
Bergstra, J. and Bengio, Y.
\newblock Random search for hyper-parameter optimization.
\newblock \emph{Journal of Machine Learning Research}, 13:\penalty0 281–305,
  2012.

\bibitem[Bi et~al.(2020{\natexlab{a}})Bi, Hu, Xie, Chen, Wei, and
  Tian]{bi2020stabilizing}
Bi, K., Hu, C., Xie, L., Chen, X., Wei, L., and Tian, Q.
\newblock Stabilizing darts with amended gradient estimation on architectural
  parameters.
\newblock \emph{CoRR}, abs/1910.11831, 2020{\natexlab{a}}.

\bibitem[Bi et~al.(2020{\natexlab{b}})Bi, Xie, Chen, Wei, and
  Tian]{bi2020goldnas}
Bi, K., Xie, L., Chen, X., Wei, L., and Tian, Q.
\newblock Gold-nas: Gradual, one-level, differentiable.
\newblock \emph{CoRR}, abs/2007.03331, 2020{\natexlab{b}}.

\bibitem[Bruna et~al.(2014)Bruna, Zaremba, Szlam, and LeCun]{bruna2014spectral}
Bruna, J., Zaremba, W., Szlam, A., and LeCun, Y.
\newblock Spectral networks and locally connected networks on graphs.
\newblock In \emph{Proc.~ICLR}, 2014.

\bibitem[Candes et~al.(2006)Candes, Romberg, and Tao]{Candes2006}
Candes, E., Romberg, J., and Tao, T.
\newblock Robust uncertainty principles: exact signal reconstruction from
  highly incomplete frequency information.
\newblock \emph{IEEE Transactions on Information Theory (TIT)}, 52\penalty0
  (2):\penalty0 489--509, 2006.

\bibitem[Chen et~al.(2021)Chen, Gong, and Wang]{chen2021neural}
Chen, W., Gong, X., and Wang, Z.
\newblock Neural architecture search on imagenet in four gpu hours: A
  theoretically inspired perspective.
\newblock \emph{arXiv preprint arXiv:2102.11535}, 2021.

\bibitem[Chen \& Hsieh(2020)Chen and Hsieh]{chen2021stabilizing}
Chen, X. and Hsieh, C.
\newblock Stabilizing differentiable architecture search via perturbation-based
  regularization.
\newblock In \emph{Proc.~ICML}, 2020.

\bibitem[Defferrard et~al.(2016)Defferrard, Bresson, and
  Vandergheynst]{defferrard2017convolutional}
Defferrard, M., Bresson, X., and Vandergheynst, P.
\newblock Convolutional neural networks on graphs with fast localized spectral
  filtering.
\newblock In \emph{Proc.~NIPS}, 2016.

\bibitem[Fey \& Lenssen(2019)Fey and Lenssen]{Fey/Lenssen/2019}
Fey, M. and Lenssen, J.~E.
\newblock Fast graph representation learning with {PyTorch Geometric}.
\newblock In \emph{ICLR Workshop}, 2019.

\bibitem[Fischer(2011)]{Fischer2011}
Fischer, H.
\newblock \emph{A history of the central limit theorem: from classical to
  modern probability theory}.
\newblock Springer, 2011.

\bibitem[Fu(1998)]{fu1998penalized}
Fu, W.~J.
\newblock Penalized regressions: the bridge versus the lasso.
\newblock \emph{Journal of Computational and Graphical Statistics}, 7\penalty0
  (3):\penalty0 397--416, 1998.

\bibitem[Gao et~al.(2020)Gao, Yang, Zhang, Zhou, and Hu]{gao2020graph}
Gao, Y., Yang, H., Zhang, P., Zhou, C., and Hu, Y.
\newblock Graph neural architecture search.
\newblock In \emph{Proc.~IJCAI}, 2020.

\bibitem[Garg et~al.(2021)Garg, Tosatto, Pan, White, and
  Mahmood]{garg2021alternate}
Garg, S., Tosatto, S., Pan, Y., White, M., and Mahmood, A.~R.
\newblock An alternate policy gradient estimator for softmax policies.
\newblock \emph{arXiv preprint arXiv:2112.11622}, 2021.

\bibitem[Giles et~al.(1998)Giles, Bollacker, and Lawrence]{Citeseer}
Giles, C.~L., Bollacker, K.~D., and Lawrence, S.
\newblock Citeseer: An automatic citation indexing system.
\newblock In \emph{Proc of Conference on Digital Libraries}, pp.\  89–98,
  1998.

\bibitem[Gunasekar et~al.(2018)Gunasekar, Lee, Soudry, and
  Srebro]{NEURIPS2018Suriya}
Gunasekar, S., Lee, J.~D., Soudry, D., and Srebro, N.
\newblock Implicit bias of gradient descent on linear convolutional networks.
\newblock In Bengio, S., Wallach, H., Larochelle, H., Grauman, K.,
  Cesa-Bianchi, N., and Garnett, R. (eds.), \emph{Proc.~NIPS}, volume~31, 2018.

\bibitem[Guo et~al.(2020{\natexlab{a}})Guo, Lin, Li, Tian, Sun, Sheng, and
  Yan]{guo2020powering}
Guo, R., Lin, C., Li, C., Tian, K., Sun, M., Sheng, L., and Yan, J.
\newblock Powering one-shot topological {NAS} with stabilized share-parameter
  proxy.
\newblock In \emph{Proc.~ECCV}, 2020{\natexlab{a}}.

\bibitem[Guo et~al.(2020{\natexlab{b}})Guo, Zhang, Mu, Heng, Liu, Wei, and
  Sun]{guo2020single}
Guo, Z., Zhang, X., Mu, H., Heng, W., Liu, Z., Wei, Y., and Sun, J.
\newblock Single path one-shot neural architecture search with uniform
  sampling.
\newblock In \emph{Proc.~ECCV}, 2020{\natexlab{b}}.

\bibitem[Hamilton(2020)]{gnn2020_Hamilton}
Hamilton, W.~L.
\newblock Graph representation learning.
\newblock \emph{Synthesis Lectures on Artificial Intelligence and Machine
  Learning}, 14\penalty0 (3):\penalty0 1--159, 2020.

\bibitem[Hamilton et~al.(2017)Hamilton, Ying, and
  Leskovec]{hamilton2018inductive}
Hamilton, W.~L., Ying, Z., and Leskovec, J.
\newblock Inductive representation learning on large graphs.
\newblock In \emph{Proc.~NIPS}, 2017.

\bibitem[Hu et~al.(2020)Hu, Fey, Zitnik, Dong, Ren, Liu, Catasta, and
  Leskovec]{hu2020open}
Hu, W., Fey, M., Zitnik, M., Dong, Y., Ren, H., Liu, B., Catasta, M., and
  Leskovec, J.
\newblock Open graph benchmark: Datasets for machine learning on graphs.
\newblock \emph{Proc.~NIPS}, 33:\penalty0 22118--22133, 2020.

\bibitem[Jones et~al.(1998)Jones, Schonlau, and Welch]{Jones1998}
Jones, D.~R., Schonlau, M., and Welch, W.~J.
\newblock Efficient global optimization of expensive black-box functions.
\newblock \emph{Journal of Global Optimization}, 13\penalty0 (4), 1998.

\bibitem[Jozefowicz et~al.(2015)Jozefowicz, Zaremba, and
  Sutskever]{jozefowicz2015empirical}
Jozefowicz, R., Zaremba, W., and Sutskever, I.
\newblock An empirical exploration of recurrent network architectures.
\newblock In \emph{Proc.~ICML}, 2015.

\bibitem[Kipf \& Welling(2016)Kipf and Welling]{Cora}
Kipf, T.~N. and Welling, M.
\newblock Semi-supervised classification with graph convolutional networks.
\newblock \emph{arXiv preprint arXiv:1609.02907}, 2016.

\bibitem[Kipf \& Welling(2017{\natexlab{a}})Kipf and Welling]{kipf2016gcn}
Kipf, T.~N. and Welling, M.
\newblock Semi-supervised classification with graph convolutional networks.
\newblock In \emph{Proc.~ICLR}, 2017{\natexlab{a}}.

\bibitem[Kipf \& Welling(2017{\natexlab{b}})Kipf and
  Welling]{kipf2017semisupervised}
Kipf, T.~N. and Welling, M.
\newblock Semi-supervised classification with graph convolutional networks.
\newblock In \emph{Proc.~ICLR}, 2017{\natexlab{b}}.

\bibitem[Liang et~al.(2019)Liang, Zhang, Sun, He, Huang, Zhuang, and
  Li]{liang2020darts}
Liang, H., Zhang, S., Sun, J., He, X., Huang, W., Zhuang, K., and Li, Z.
\newblock {DARTS+:} improved differentiable architecture search with early
  stopping.
\newblock \emph{CoRR}, abs/1909.06035, 2019.

\bibitem[Liu et~al.(2019{\natexlab{a}})Liu, Simonyan, and Yang]{liu2019darts}
Liu, H., Simonyan, K., and Yang, Y.
\newblock {DARTS:} differentiable architecture search.
\newblock In \emph{Proc.~ICLR}, 2019{\natexlab{a}}.

\bibitem[Liu et~al.(2019{\natexlab{b}})Liu, Chen, Li, Zhou, Li, Song, and
  Qi]{liu2018geniepath}
Liu, Z., Chen, C., Li, L., Zhou, J., Li, X., Song, L., and Qi, Y.
\newblock Geniepath: Graph neural networks with adaptive receptive paths.
\newblock In \emph{Proc.~AAAI}, 2019{\natexlab{b}}.

\bibitem[Luo et~al.(2018)Luo, Tian, Qin, Chen, and Liu]{luo2019neural}
Luo, R., Tian, F., Qin, T., Chen, E., and Liu, T.
\newblock Neural architecture optimization.
\newblock In \emph{Proc.~NIPS}, 2018.

\bibitem[McAuley et~al.(2015)McAuley, Targett, Shi, and Van
  Den~Hengel]{mcauley2015image}
McAuley, J., Targett, C., Shi, Q., and Van Den~Hengel, A.
\newblock Image-based recommendations on styles and substitutes.
\newblock In \emph{Proc.~SIGIR}, pp.\  43--52, 2015.

\bibitem[Nguyen et~al.(2018)Nguyen, Wong, and Hegde]{nguyen2018provable}
Nguyen, T., Wong, R., and Hegde, C.
\newblock A provable approach for double-sparse coding.
\newblock In \emph{Proc.~AAAI}, volume~32, 2018.

\bibitem[Olshausen \& Field(1997)Olshausen and Field]{OLSHAUSEN1997}
Olshausen, B.~A. and Field, D.~J.
\newblock Sparse coding with an overcomplete basis set: A strategy employed by
  v1?
\newblock \emph{Vision Research}, 37\penalty0 (23):\penalty0 3311--3325, 1997.

\bibitem[Paszke et~al.(2019)Paszke, Gross, Massa, Lerer, Bradbury, Chanan,
  Killeen, Lin, Gimelshein, Antiga, et~al.]{NEURIPS2019_9015}
Paszke, A., Gross, S., Massa, F., Lerer, A., Bradbury, J., Chanan, G., Killeen,
  T., Lin, Z., Gimelshein, N., Antiga, L., et~al.
\newblock Pytorch: An imperative style, high-performance deep learning library.
\newblock \emph{Proc.~NIPS}, 32, 2019.

\bibitem[Qin et~al.(2022)Qin, Zhang, Wang, Zhang, and Zhu]{qin2022bench}
Qin, Y., Zhang, Z., Wang, X., Zhang, Z., and Zhu, W.
\newblock Nas-bench-graph: Benchmarking graph neural architecture search.
\newblock \emph{arXiv preprint arXiv:2206.09166}, 2022.

\bibitem[Ranstam \& Cook(2018)Ranstam and Cook]{ranstam2018lasso}
Ranstam, J. and Cook, J.
\newblock Lasso regression.
\newblock \emph{Journal of British Surgery}, 105\penalty0 (10):\penalty0
  1348--1348, 2018.

\bibitem[Schmidt et~al.(2009)Schmidt, Fung, and
  Rosales]{schmidt2009optimization}
Schmidt, M., Fung, G., and Rosales, R.
\newblock Optimization methods for l1-regularization.
\newblock \emph{University of British Columbia, Technical Report TR-2009-19},
  2009.

\bibitem[Sen et~al.(2008)Sen, Namata, Bilgic, Getoor, Galligher, and
  Eliassi-Rad]{PubMed}
Sen, P., Namata, G., Bilgic, M., Getoor, L., Galligher, B., and Eliassi-Rad, T.
\newblock Collective classification in network data.
\newblock \emph{AI Magazine}, 29\penalty0 (3):\penalty0 93, 2008.

\bibitem[Shevade \& Keerthi(2003)Shevade and Keerthi]{shevade2003simple}
Shevade, S.~K. and Keerthi, S.~S.
\newblock A simple and efficient algorithm for gene selection using sparse
  logistic regression.
\newblock \emph{Bioinformatics}, 19\penalty0 (17):\penalty0 2246--2253, 2003.

\bibitem[Shu et~al.(2021)Shu, Cai, Dai, Ooi, and Low]{shu2021nasi}
Shu, Y., Cai, S., Dai, Z., Ooi, B.~C., and Low, B. K.~H.
\newblock Nasi: Label-and data-agnostic neural architecture search at
  initialization.
\newblock \emph{arXiv preprint arXiv:2109.00817}, 2021.

\bibitem[Su et~al.(2021)Su, You, Zheng, Wang, Qian, Zhang, and Xu]{su2021kshot}
Su, X., You, S., Zheng, M., Wang, F., Qian, C., Zhang, C., and Xu, C.
\newblock K-shot {NAS:} learnable weight-sharing for {NAS} with k-shot
  supernets.
\newblock In \emph{Proc.~ICML}, 2021.

\bibitem[ter Haar(1949)]{Khinchin}
ter Haar, D.
\newblock Mathematical foundations of statistical mechanics. a. i. khinchin.
\newblock \emph{Science}, 110\penalty0 (2865):\penalty0 570--570, 1949.

\bibitem[Tropp(2004)]{Tropp2004}
Tropp, J.
\newblock Greed is good: algorithmic results for sparse approximation.
\newblock \emph{IEEE Transactions on Information Theory (TIT)}, 50\penalty0
  (10):\penalty0 2231--2242, 2004.

\bibitem[Velickovic et~al.(2018)Velickovic, Cucurull, Casanova, Romero,
  Li{\`{o}}, and Bengio]{Petar2018graph}
Velickovic, P., Cucurull, G., Casanova, A., Romero, A., Li{\`{o}}, P., and
  Bengio, Y.
\newblock Graph attention networks.
\newblock In \emph{Proc.~ICLR}, 2018.

\bibitem[Wu et~al.(2019{\natexlab{a}})Wu, Dai, Zhang, Wang, Sun, Wu, Tian,
  Vajda, Jia, and Keutzer]{wu2019fbnet}
Wu, B., Dai, X., Zhang, P., Wang, Y., Sun, F., Wu, Y., Tian, Y., Vajda, P.,
  Jia, Y., and Keutzer, K.
\newblock Fbnet: Hardware-aware efficient convnet design via differentiable
  neural architecture search.
\newblock In \emph{Proc.~CVPR}, 2019{\natexlab{a}}.

\bibitem[Wu et~al.(2019{\natexlab{b}})Wu, Souza, Zhang, Fifty, Yu, and
  Weinberger]{wu2019sgc}
Wu, F., Souza, A., Zhang, T., Fifty, C., Yu, T., and Weinberger, K.
\newblock Simplifying graph convolutional networks.
\newblock In \emph{Proc.~ICML}, pp.\  6861--6871, 2019{\natexlab{b}}.

\bibitem[Wu et~al.(2021)Wu, Pan, Chen, Long, Zhang, and Yu]{2021GNNs}
Wu, Z., Pan, S., Chen, F., Long, G., Zhang, C., and Yu, P.~S.
\newblock A comprehensive survey on graph neural networks.
\newblock \emph{IEEE Transactions on Neural Networks and Learning Systems},
  32\penalty0 (1):\penalty0 4–24, 2021.

\bibitem[Xie \& Yuille(2017)Xie and Yuille]{xie2017genetic}
Xie, L. and Yuille, A.~L.
\newblock Genetic {CNN}.
\newblock In \emph{Proc.~ICCV}, 2017.

\bibitem[Xie et~al.(2022)Xie, Chen, Bi, Wei, Xu, Wang, Chen, Xiao, Chang,
  Zhang, and Tian]{xie2020weightsharing}
Xie, L., Chen, X., Bi, K., Wei, L., Xu, Y., Wang, L., Chen, Z., Xiao, A.,
  Chang, J., Zhang, X., and Tian, Q.
\newblock Weight-sharing neural architecture search: {A} battle to shrink the
  optimization gap.
\newblock \emph{ACM Computing Surveys}, 54\penalty0 (9):\penalty0
  183:1--183:37, 2022.

\bibitem[Xu et~al.(2018)Xu, Li, Tian, Sonobe, Kawarabayashi, and
  Jegelka]{pmlr-v80-xu18c}
Xu, K., Li, C., Tian, Y., Sonobe, T., Kawarabayashi, K.-i., and Jegelka, S.
\newblock Representation learning on graphs with jumping knowledge networks.
\newblock In \emph{Proc.~ICML}, volume~80, pp.\  5453--5462, 2018.

\bibitem[Xu et~al.(2019)Xu, Hu, Leskovec, and Jegelka]{xu2019powerful}
Xu, K., Hu, W., Leskovec, J., and Jegelka, S.
\newblock How powerful are graph neural networks?
\newblock In \emph{Proc.~ICML}, 2019.

\bibitem[You et~al.(2020)You, Ying, and Leskovec]{you2020design}
You, J., Ying, Z., and Leskovec, J.
\newblock Design space for graph neural networks.
\newblock \emph{Proc.~NIPS}, 33, 2020.

\bibitem[Yu \& Zhu(2020)Yu and Zhu]{yu2020hpo}
Yu, T. and Zhu, H.
\newblock Hyper-parameter optimization: {A} review of algorithms and
  applications.
\newblock \emph{CoRR}, abs/2003.05689, 2020.

\bibitem[Zela et~al.(2020)Zela, Elsken, Saikia, Marrakchi, Brox, and
  Hutter]{zela2020understanding}
Zela, A., Elsken, T., Saikia, T., Marrakchi, Y., Brox, T., and Hutter, F.
\newblock Understanding and robustifying differentiable architecture search.
\newblock In \emph{Proc.~ICLR}, 2020.

\bibitem[Zhang et~al.(2021{\natexlab{a}})Zhang, Huang, Wang, Xiang, and
  Pan]{zhang2018search}
Zhang, X., Huang, Z., Wang, N., Xiang, S., and Pan, C.
\newblock You only search once: Single shot neural architecture search via
  direct sparse optimization.
\newblock \emph{IEEE TPAMI}, 43\penalty0 (9):\penalty0 2891--2904,
  2021{\natexlab{a}}.

\bibitem[Zhang et~al.(2015)Zhang, Xu, Yang, Li, and Zhang]{zhang2015survey}
Zhang, Z., Xu, Y., Yang, J., Li, X., and Zhang, D.
\newblock A survey of sparse representation: algorithms and applications.
\newblock \emph{IEEE Access}, 3:\penalty0 490--530, 2015.

\bibitem[Zhang et~al.(2021{\natexlab{b}})Zhang, Wang, and
  Zhu]{zhang2021automated}
Zhang, Z., Wang, X., and Zhu, W.
\newblock Automated machine learning on graphs: {A} survey.
\newblock In \emph{Proc.~IJCAI}, 2021{\natexlab{b}}.

\bibitem[Zhao et~al.(2021{\natexlab{a}})Zhao, Wei, quanming yao, and
  He]{zhao2021efficient}
Zhao, H., Wei, L., quanming yao, and He, Z.
\newblock Efficient graph neural architecture search, 2021{\natexlab{a}}.
\newblock URL \url{https://openreview.net/forum?id=IjIzIOkK2D6}.

\bibitem[Zhao et~al.(2021{\natexlab{b}})Zhao, Yao, and Tu]{zhao2021search}
Zhao, H., Yao, Q., and Tu, W.
\newblock Search to aggregate neighborhood for graph neural network.
\newblock In \emph{Proc.~ICDE}, 2021{\natexlab{b}}.

\bibitem[Zhao et~al.(2020)Zhao, Wang, Gao, Mullins, Lio, and
  Jamnik]{zhao2020probabilistic}
Zhao, Y., Wang, D., Gao, X., Mullins, R., Lio, P., and Jamnik, M.
\newblock Probabilistic dual network architecture search on graphs.
\newblock \emph{arXiv preprint arXiv:2003.09676}, 2020.

\bibitem[Zhou et~al.(2019)Zhou, Song, Huang, and Hu]{zhou2019autognn}
Zhou, K., Song, Q., Huang, X., and Hu, X.
\newblock Auto-gnn: Neural architecture search of graph neural networks.
\newblock \emph{CoRR}, abs/1909.03184, 2019.

\bibitem[Zoph \& Le(2017)Zoph and Le]{DBLP:conf/iclr/ZophL17}
Zoph, B. and Le, Q.~V.
\newblock Neural architecture search with reinforcement learning.
\newblock In \emph{Proc.~ICLR}, 2017.

\end{thebibliography}

%%%%%%%%%%%%%%%%%%%%%%%%%%%%%%%%%%%%%%%%%%%%%%%%%%%%%%%%%%%%%%%%%%%%%%%%%%%%%%%
%%%%%%%%%%%%%%%%%%%%%%%%%%%%%%%%%%%%%%%%%%%%%%%%%%%%%%%%%%%%%%%%%%%%%%%%%%%%%%%
% APPENDIX
%%%%%%%%%%%%%%%%%%%%%%%%%%%%%%%%%%%%%%%%%%%%%%%%%%%%%%%%%%%%%%%%%%%%%%%%%%%%%%%
%%%%%%%%%%%%%%%%%%%%%%%%%%%%%%%%%%%%%%%%%%%%%%%%%%%%%%%%%%%%%%%%%%%%%%%%%%%%%%%
% \iffalse
\newpage
\appendix
\onecolumn
\newpage
\section{Appendix}
\subsection{Dataset Details}
\label{subsec:dataset}
In the transductive learning task, node features and edges between nodes in the whole datasets were known beforehand. We learn from the already labeled training dataset and then predict the labels of the testing dataset. Three benchmark datasets were used in this setting: CiteSeerr~\citep{Citeseer}, Cora~\citep{Cora} and PubMed~\citep{PubMed}. All of the three benchmark datasets are citation networks. In citation networks, each node represents a work, and each edge shows the relationship between two papers in terms of citations. The datasets contain bag-of-words features for each node, and the goal is to categorize papers into various subjects based on the citation. In addition to the three benchmark datasets, we also employ two another datasets: Amazon Computers~\citep{mcauley2015image} and ogbn-arxiv dataset~\citep{hu2020open}. Amazon Computers is a subset of the Amazon co-purchase graph, where nodes represent commodities and edges connect them if they are frequently purchased together. Product reviews are encoded as bag-of-word feature vectors in node features, and class labels are assigned based on product category. The ogbn-arxiv dataset~\citep{hu2020open} is a directed graph that represents the citation network of all computer-related papers on the arxiv website. Each node is an arxiv paper, and each edge represents an article citing another article.

In the inductive task, several graphs are used as training data, while other completely unseen graphs are used as validation and test data. For the inductive setting, we use the PPI~\cite{hamilton2018inductive} dataset as benchmark dataset. The task of PPI dataset is to classify the different protein functions. It consists of 24 graphs, with each representing a different human tissue. Each node has properties such as positional gene sets, motif gene sets, and immunological signatures, ant gene ontology sets as labels. Twenty graphs are chosen for training, two graphs for validation, and the remaining for testing. For the inductive task, we use Micro-F1 as the evaluation metric.

The used datasets are concluded in \Cref{table:dataset}.

\begin{table}[!ht]
\centering
\caption{\footnotesize The Statistics of Datasets}

\footnotesize
\begin{tabular}{l|ccccc|c}
\toprule
    ~ & \multicolumn{5}{c}{Transductive} & Inductive\\\cline{2-7}
    ~ & CiteSeer & Cora & PubMed & Computers & ogbn-arxiv & PPI\\
    \midrule
    $\#$nodes & 2708 & 3327 & 19717 & 13752  & 169343 & 56944\\
    $\#$edges & 5278 & 4552 & 44324 & 245861 & 1166243 & 818716\\
    $\#$features & 1433 & 3703 & 500 & 767 & 128 & 121\\
    $\#$classes  & 7 & 6 & 3 & 10 & 40 & 50\\
    \bottomrule
\end{tabular}
\label{table:dataset}
\end{table}

\subsection{Proof of Analysis for No-update Scheme in GNNs}
~\label{theo:no-update_scheme}
\begin{theorem}
\label{theo:appdix_theo1}
Assume $\bm{W}_l(0)$ is randomly initialized for all $l\in [1, L]$, if  $\prod_{l=1}^{L}\bm{W}_l(0)$ is full rank, there must exist a weight matrix for the output layer, i.e., $\tilde{\bm{W}_o}$, that makes the final output the same as the one from a well-trained network:  
\begin{equation}
\label{eq:appdix_theo}
    \bm{A}^L\bm{X}\prod_{l=1}^{L}\bm{W}_l^*\bm{W}_{o}^*=\bm{A}^L\bm{X}\prod_{l=1}^{L}\bm{W}_l(0)\tilde{\bm{W}_{o}}.
\end{equation}
\end{theorem}
\begin{proof}
Because $\prod_{l=1}^{L}\bm{W}_l(0)$ is a full-rank matrix, $\prod_{l=1}^{L}\bm{W}_l(0)$ is invertible. We can define $\tilde{\bm{W}_{o}}$ by 
\begin{equation}\label{eq:tW_o}
    \tilde{\bm{W}_{o}} \defeq (\prod_{l=1}^{L}\bm{W}_l(0))^{-1}\prod_{l=1}^{L}\bm{W}_l^*\bm{W}_{o}^*,
\end{equation}
Hence, by multiplying $\prod_{l=1}^{L}\bm{W}_l(0)$ and $\bm{A}^L\bm{X}$ on both sides in \Cref{eq:tW_o}, we can attain \Cref{eq:appdix_theo}.
\end{proof}

\begin{theorem}
% ~\label{eq:theo2}
Assume a GNN model has either the exponential loss or the log-loss, the desired weight  $\tilde{\bm{W}_o}$ is secured when updating with gradient descent and  initializing $\prod_{l=1}^{L}\bm{W}_l(0)$ as an orthogonal matrix. Mathematically, we have $\hat{\bm{W}_{o}}(+\infty)=\hat{\bm{W}_{o}}^*=\tilde{\bm{W}_{o}}$. 
\end{theorem}

\begin{proof}
We first frame the simplified GNNs as $\bm{A}^L\bm{X}\prod_{l=1}^{L}\bm{W}_l=(\bm{A}^L\bm{X})(\prod_{l=1}^{L}\bm{W}_l)$. According to \Cref{theo:theo2}, we define the corresponding $\bm{\beta}$ as $\bm{\beta}=\prod_{l=1}^{L}{\bm{W}_l}$. 

When using the gradient descent to optimize $\bm{W}_l$ and $\bm{W}_o$, we have the optimized $\bm{\beta}$ as $\bm{\beta}^*=\prod_{l=1}^{L}\bm{W}_l^*\cdot \bm{W}_o^*$. The obtained $\bm{\beta}^*$ can be viewed as the max-margin hyperplane to separate data $\bm{A}^L\bm{X}$ according to \Cref{theo:theo2}. Omitting the output layer, we let the  untrained GNN model as $\bm{A}^L\bm{X}\prod_{l=1}^{L}\bm{W}_l(0)=\bm{A}^L\bm{X}\bm{O}$ and assume $\bm{O}=\prod_{l=1}^{L}\bm{W}_l(0)$ is  an orthogonal square matrix.
  
Denoting a complete GNN model as $\bm{A}^L\bm{X}\bm{O}\bm{W}_o$, the optimized output weight, i.e. $\hat{\bm{W}}_o^*$, becomes the max-margin separating hyperplane of $\bm{A}^L\bm{X}\bm{O}$ because it is also trained by gradient descent. Here, we fix the $\bm{W}_l(0)$ during training to maintain its orthogonality. We know that the max-margin hyperplane of any data remains the same if and only if it takes orthogonal transformations, resulting in the following equivalence: $\bm{O}\hat{\bm{W}_{o}}^* =\bm{\beta}^*=\prod_{l=1}^{L}\bm{W}_l^*\cdot \bm{W}_o^*$. Finally, we get $\hat{\bm{W}_{o}}^* = \bm{O}^{-1}\prod_{l=1}^{L}\bm{W}_l^*\bm{W}_o^*
    = (\prod_{l=1}^{L}\bm{W}_l(0))^{-1}\prod_{l=1}^{L}\bm{W}_l^*\bm{W}_o^*
    = \tilde{\bm{W}_o}$.
\end{proof}

\subsection{Proof for the Unified Format of Search Space in NAC}
\label{proof:fixed}
\begin{theorem}
% ~\label{theo:gnn_layer_format}
Simplifying the complex non-linear functions, each GNN layer in the search space can be unified by $\sum_k\mathrm{P}_{l}^k(L)\bm{H}_{l}\bm{W}_{l}^{k}$, where $\mathrm{P}(\cdot)$ denotes the fixed polynomial function with subscripts indicating individual terms, and superscript $(l)$ marks the current layer.
\end{theorem}

\begin{proof}
Without any loss of generality, we omit the lowerscript $(l)$ of $\sum_s C^{(s)} \bm{H}_{l} \bm{W}_{l}^{(s)}$ to  prove the equivalence between $\sum_s C^{( s)}\bm{H}\bm{W}^{(s)}$ and $\sum_k\mathrm{P}^k(\bm{L})\bm{H}\bm{W}^k$ for arbitrary layer. 

According to Table 1 in \cite{balcilar2021bridge}, we have the frequency support for most GNNs and except for GAT as 
\begin{align}
    \label{eq:appendix_uniform_format_gnn}
    C^{(s)} = \mathrm{P}^{(s)}(\hat{\bm{L}}),
\end{align}
where $\mathrm{P}^{(s)}(\cdot)$ denotes certain polynomial functions
and $\hat{L}$ is a re-scaling and shifting of the original/normalized/re-normalized Laplacian matrix, which is universally denoted as $L$ in this proof and its \Cref{theo:search_space_format}.
\footnote{This symbolic simplification does not affect the conclusion of \Cref{theo:gnn_layer_format} and \Cref{theo:search_space_format}. Because both claims aim to achieve a fixed dictionary format of GNNs, the different expressions of $L$ do not change the format's nature of being fixed.}
Equivalently, we have:
\begin{align}
    \label{eq:rescale_l}
    \hat{\bm{L}} = \mathrm{P}_1'(\bm{L})
\end{align}
Putting \Cref{eq:rescale_l} into \Cref{eq:appendix_uniform_format_gnn}, it produces:
\begin{align}
    C^{(s)} &= \mathrm{P}^{(s)}(\mathrm{P}_1'(\bm{L})) \\
    &= \mathrm{P}'^{(s)}(\bm{L}).
\end{align}
And we have:
\begin{align}
    \sum_s C^{(s)}\bm{H}\bm{W}^{(s)} &= \sum_s \mathrm{P}'^{(s)}(\bm{L})\bm{H}\bm{W}^{(s)}.
\end{align}

Without loss of generality, suppose $s=1$, $C^{(0)}=\mathrm{P}'^{(0)}(\bm{L})=\alpha \bm{L}^2 + \beta \bm{L}$ and $C^{(1)}=\mathrm{P}'^{(1)}(\bm{L})=\gamma \bm{L}$, and it gives:
\begin{align}
\label{eq_14}
    C^{(0)}\bm{H}\bm{W}^{(0)}+C^{(1)}\bm{H}\bm{W}^{(1)} 
    &= (\alpha \bm{L}^2 + \beta \bm{L})\bm{H}\bm{W}^{(0)} + \gamma \bm{L}\bm{H}\bm{W}^{(1)}\\
    &= \bm{L}^2\bm{H}(\alpha \bm{W}^{(0)}) + \bm{L}\bm{H}(\beta \bm{W}^{(0)} + \gamma \bm{W}^{(1)})
\end{align}
Replacing $\alpha \bm{W}^{(0)}$ with $\bm{W}^0$ and $\beta \bm{W}^{(0)} + \gamma \bm{W}^{(1)}$ with $\bm{W}^1$, \Cref{eq_14} is transformed to
$\bm{L}^2\bm{H}\bm{W}^0 + \bm{L}\bm{H}\bm{W}^1$.
Then, assuming $\mathrm{P}^1=\bm{L}^2$ and $\mathrm{P}^0=\bm{L}^1$, it produces $\mathrm{P}^1\bm{H}\bm{W}^0 + \mathrm{P}^0\bm{H}\bm{W}^1 = \sum_{k=0}^1\mathrm{P}^k(L)\bm{H}\bm{W}^k$.

To this end, we prove the equivalence between $\sum_s C^{( s)}\bm{H}\bm{W}^{(s)}$ and $\sum_k\mathrm{P}^k(\bm{L})\bm{H}\bm{W}^k$. 

Specially, GAT contains complex non-linear functions, and we reduce its expression to $\bm{H}_{(l+1)}=\sigma((\bm{H}_{l}\bm{W}_{l+1}^{att}{\bm{H}_{l}}^T\odot \bm{A}) \bm{H}_{l}\bm{W}_{l+1})$. In this expression, $\bm{W}_{att}$ is a attention function, and $\odot$ denotes element-wise multiplication. 
Omitting $\odot$ and $\sigma$ to obtain the format $\bm{H}_{l+1}=\bm{H}_{l}{\bm{W}_{l+1}}'$, which is identical to the case when $\mathrm{P}^{(s)}=\bm{I}$ and therefore inclusive in the discussion above.
\end{proof}

\begin{corollary}
% ~\label{theo:search_space_format}
The Search space of GNNs can be unified as $\sum_{k=0}^K\mathrm{P}^k(\bm{L})\bm{X}\bm{W}^k=\bm{D}\bm{W}$ when leaving the activation functions, where $\Vert$ stands for concatenating a set of matrices horizontally, $\bm{D}=\Vert\{\mathrm{P}^k(\bm{L})\bm{X}\}$ is the fixed base, and $\bm{W}=\Vert\{\bm{W}^k\}^T$ is the trainable parameters.
\end{corollary}

\label{proof:gnn_dict}
\begin{proof}
Suppose the node features as $\bm{H}_{0}=\bm{X}$, and a two-layer GNN consists of $P_0^{(S)},P_1^{(L)}$ and $\bm{W}_0^{(S)}, \bm{W}_1^{(L)}$. 
Here, we assume $S,L=1$ for simplicity.

Based on \Cref{theo:gnn_layer_format}, the first layer outputs $\bm{H}_{1} = \sum_{s=0}^2\mathrm{P}_0^{(s)}(\bm{L})\bm{X}\bm{W}_0^{(s)}=\mathrm{P}_0^{(0)}(\bm{L})\bm{X}\bm{W}_0^{(0)}+\mathrm{P}_0^{(1)}(\bm{L})\bm{X}\bm{W}_0^{(1)}$.
Then, we pass $H_{1}$ to the second layer the below:
\begin{equation}
\label{eq_18}
    \begin{aligned}
    H_{2}
    &= \sum_{l=0}^2 \mathrm{P}_1^{(l)}(\bm{L}) \bm{H}^{(1)} \bm{W}_1^{(l)}\\
    &= \mathrm{P}_1^{(0)}(\bm{L}) \bm{H}_1 \bm{W}_1^{(0)} + \mathrm{P}_1^{(1)}(\bm{L}) \bm{H}_1 \bm{W}_1^{(1)}\\
    &= \mathrm{P}_1^{(0)}(\bm{L}) (\mathrm{P}_0^{(0)}(\bm{L})\bm{X}\bm{W}_0^{(0)}+\mathrm{P}_0^{(1)}(\bm{L})\bm{X}\bm{W}_0^{(1)}) \bm{W}_1^{(0)} \\
    &+ \mathrm{P}_1^{(1)}(\bm{L}) (\mathrm{P}_0^{(0)}(\bm{L})\bm{X}\bm{W}_0^{(0)}+\mathrm{P}_0^{(1)}(\bm{L})\bm{X}\bm{W}_0^{(1)}) \bm{W}_1^{(1)}\\
    &= \mathrm{P}_{1\times0}^{(0\times 0)}(\bm{L})\bm{X}\bm{W}_{0\times 1}^{(0\times 0)} + \mathrm{P}^{(0\times 1)}_{1\times 0}(\bm{L}) X \bm{W}^{(1\times 0)}_{0\times 1}\\
    &+ \mathrm{P}^{(1\times 0)}_{1\times 0}(\bm{L}) \bm{X} \bm{W}^{(0\times 1)}_{0\times 1} +
    \mathrm{P}^{(1\times 1)}_{1\times 0}(\bm{L}) \bm{X} \bm{W}^{(1\times 1)}_{0\times 1} \\
    &= \sum_k\mathrm{P}^k(L) \bm{X} \bm{W}^k.
    \end{aligned}
\end{equation}

In \Cref{eq_18}, we merge the polynomial of a polynomial as $\mathrm{P}_1^{(0)}\mathrm{P}_0^{(0)}=\mathrm{P}_{1\times 0}^{(0\times 0)}$. Similarly, we replace $\bm{W}_0^{(0)}\bm{W}_1^{(0)}$ with $\bm{W}_{0\times 1}^{(0\times 0)}$. 

Ultimately, we expand \Cref{eq_18} by $\bm{D}=\Vert\{\mathrm{P}^{(0\times 0)}_{1\times 0}(\bm{L}) X, \mathrm{P}^{(0\times 1)}_{1\times 0}(\bm{L}) X, \mathrm{P}^{(1\times 0)}_{1\times 0}(\bm{L}) X, \mathrm{P}^{(1\times 1)}_{1\times 0}(\bm{L}) X\}$ and $\bm{W}=\Vert\{\bm{W}_{0\times 1}^{(0\times 0)},\bm{W}_{0\times 1}^{(1\times 0)},\bm{W}_{0\times 1}^{(0\times 1)},\bm{W}_{0\times 1}^{(1\times 1)}\}$, so that $\sum_k\mathrm{P}_k(\bm{\bm{L}}) X \bm{W}_k = \bm{D}\bm{W}$.
\end{proof}

\subsection{Proof for the Optmization of $l2$-norm objective in NAC}
\begin{theorem}
The gradient from Cross-Entropy loss function is feasible to optimize an $l2$-norm objective in a gradient descent manner, such that the initial $z^0$ converges geometrically to the optimal $z^*$ with a tiny error, e.g., 
\begin{equation}
    \lVert z^{s+1}-z^* \rVert \leq (1-2\alpha \eta)^s\lVert z^0-z^* \rVert + 2\gamma/\alpha,
\end{equation}
where $s$ signifies the current iteration, and $\gamma, \alpha,\eta$ define the numerical bound for the error.
\end{theorem}

\label{proof:equal}
\begin{proof}
Given an iterative algorithm $\mathcal{A}$ that optimizes for a solution $z^*\in\mathrm{R}^n$ for a function $f(z)$, it yields $z^{s+1} = z^s - \eta g^s$ in each step $s$. Here, $\eta$ is the step size, and $g^s$ is the gradient vector .
In the following, we give the \textit{sufficient conditions} for such $g^s$, quoted from \cite{nguyen2018provable} and \cite{arora2015simple}.

\begin{definition}
A vector $g^s$ at the $s$-th iteration is $(\alpha,\beta,\gamma_s)$-correlated with a desired solution $z^*$ if
\begin{equation}
    \langle g^s,z^s-z^* \rangle \geq \alpha \lVert z^s - z^* \rVert^2 + \beta \lVert g^s \rVert^2 - \gamma_s.
\end{equation}
\label{defi:correlated}
\end{definition}

Under the convex optimization, if $f$ is $2\alpha$-strong convex and $1/2\beta$-smooth, and $g^s$ is set as $\nabla_z f(z)$, then $g^s$ is $(\alpha, \beta, 0)$-correlated with $z^*$.
\begin{theorem}
(Convergence of approximate gradient descent). Suppose that $g^s$ satisfies the conditions described in \Cref{defi:correlated} for $s=1,2,\cdots,T$. Moreover, $0<\eta\leq 2\beta$ and $\gamma=\max_{s=1}^T \gamma_s$. Then, the following holds for all $s$:
\begin{equation}
    \lVert z^{s+1} - z^{s} \rVert^2 \leq (1-2\alpha\eta)\lVert z^{s} - z^{*} \rVert^2 + 2\eta\gamma_s.
\end{equation}
In particular, the above update converges geometrically to $z^*$ with an error:
\begin{equation}
    \lVert z^{s+1} - z^{*} \rVert^2 \leq (1-2\alpha\eta)^s\lVert z^{0} - z^{*} \rVert^2 + 2\gamma/\alpha.
\end{equation}
\end{theorem}
So, if the cross-entropy loss function satisfies the conditions of convexity and smoothness in \Cref{defi:correlated}, it meets the \textit{sufficient conditions} for convergence.

Recall that the cross-entropy loss is defined as:
\begin{equation}
    \begin{aligned}
    g(W)&=-\frac{1}{m}\sum_{i=1}^m\sum_{j=1}^c Y_{ij}\log ( \frac{\exp F_i W_j}{\sum_{k=1}^c \exp F_i W_k} )\\
    &=-\frac{1}{m}\sum_{i=1}^m\sum_{j=1}^c Y_{ij} F_i W_j + \frac{1}{m}\sum_{i=1}^m \log \sum_{k=1}^c \exp F_i W_k
    \end{aligned}
\end{equation}
where $Y$ is one-hot matrix with $c$ classes. Then, we give the gradient w.r.t $W_j$:
\begin{equation}
    \begin{aligned}
    \nabla_{W_j} g = -\frac{1}{m}\sum_{i=1}^m [Y_{ij}=1]F_i + \frac{1}{m}\sum_{i=1}^m \frac{\exp F_i W_j}{\sum_{k=1}^c \exp F_i W_k} F_i,
    \end{aligned}
\end{equation}
where we symbolize the softmax function $s_i = \frac{\exp F_i W_j}{\sum_{k=1}^c \exp F_i W_k}$, and thus:
\begin{equation}
    \begin{aligned}
    \nabla_{W_j} g = \frac{1}{m}\sum_{i=1}^m ( s_i-[Y_{ij}=1] ) F_i.
    \end{aligned}
\end{equation}

{\bf Smoothness:} In the extreme case where
all the softmax values are equal, i.e., $s_i = \frac{1}{c}$ and $Y_{ij}=1$, we can develop an upper bound as follows, % of the absolute value:
\begin{equation}
    |s_i-[Y_{ij}=1]|\leq\frac{c-1}{c}.
\end{equation}
Based on this, we get:
\begin{equation}
\label{eq_27}
    \begin{aligned}
    |\nabla_{W_j} g| &\leq \frac{c-1}{mc} |\sum_{i=1}^m F_i|
    \end{aligned}
\end{equation}
This gives us the Lipschitz constant that $L=\frac{c-1}{mc}\lVert F \rVert$. Sufficiently, there exists $\beta^*$ that the cross-entropy loss function is $\beta^*$-smooth.

{\bf Convexity:}
Given $W_k,W_j,F_i\in\mathbb{R}^d$, we assume each element sampled from a certain normal distribution. Without loss of generality, we give $W_k\sim \mathcal{N}(\mu_1, \sigma_1^2), W_k\sim \mathcal{N}(\mu_2, \sigma_2^2)$, and $F_i\sim \mathcal{N}(\mu, \sigma^2)$. 

In particular, we take $F_i$ as an example and the other two follow the same rule:
\begin{equation}
\label{eq_28}
    \begin{aligned}
    \mathrm{E}[F_i^2]&=\mathrm{E}[(F_i-\mu)^2+2F_i\mu-\mu^2]\\
    &= \sigma^2 + \mu^2,
    \end{aligned}
\end{equation}
which equals to $\frac{1}{d}$, and thus, $\mu^2 = \frac{1}{d} - \sigma^2$. For $W_k,W_j$, we have  $\mu_1^2 = \frac{1}{d} - \sigma_1^2$, and $\mu_2^2 = \frac{1}{d} - \sigma_2^2$. For strong convexity, we need to demonstrate there exists $\alpha\in\mathbb{R}$ supporting 
$\langle\nabla_{W_k}g - \nabla_{W_k}g, W_k - W_j\rangle \geq \alpha \langle  W_k - W_j,  W_k - W_j \rangle$. 
By taking the \Cref{eq_27}, we can analyze the left-hand side as:
\begin{equation}
\label{eq_29}
    \begin{aligned}
    \langle\nabla_{W_k}g - \nabla_{W_k}g, W_k - W_j\rangle 
    &= \frac{1}{m}\sum_{i=1}^m [Y_{ij}\neq Y_{ik}] \langle F_i, W_k - W_j\rangle \\
    & \geq \langle \frac{1}{m}\sum_{i=1}^m [Y_{ij}\neq Y_{ik}]  F_i, W_k - W_j\rangle \\
    & \geq \frac{2(c-1)}{mc^2} \sum_{i=1}^m \langle F_i, W_k - W_j \rangle .
    \end{aligned}
\end{equation}

In \Cref{eq_29}, we assume $Y_i$ obeys uniform distribution that $P(Y_{ik}=1) = \frac{1}{c}$, then we replace $[Y_{ij}\neq Y_{ik}]$ by its average $\frac{2(c-1)}{c^2}$, identical to the probability, $P([Y_{ij}\neq Y_{ik}])$. 

Next, we suppose $\frac{1}{d}\sum_{h=1}{d} F_i$ is an unbiased estimation of its expectation, and same for $W_j, W_k$. Here, we replace $W_j-W_k$ with a new variable as $W'\sim \mathcal{N}(\mu_1-\mu_2, \sigma^2_1+\sigma^2_2)$. With the help of \Cref{eq_29}, we get,
\begin{equation}
    \begin{aligned}
    \langle\nabla_{W_k}g - \nabla_{W_k}g, W_k - W_j\rangle 
    &\geq \frac{2(c-1)}{c^2}  d \mu(\mu_1 - \mu_2).
    \end{aligned}
\end{equation}

By taking these into the right-hand side $\langle  W_k - W_j,  W_k - W_j \rangle$, we have: % $d \mathrm{E}[W'^2]$:
\begin{equation}
    \begin{aligned}
    d \mathrm{E}[W'^2]
    &= d \mathrm{E}[ ( W'- (\mu_1-\mu_2))^2 -(\mu_1-\mu_2)^2 + 2W'(\mu_1-\mu_2) ]\\
    &= d ( \mathrm{E} [ ( W'- (\mu_1-\mu_2)) ^2]-(\mu_1-\mu_2)^2 + 2(\mu_1-\mu_2)\mathrm{E}[W'] ) \\
    &= d( \sigma_1^2 + \sigma_2^2 + (\mu_1 - \mu_2)^2 )
    \end{aligned}     
\end{equation}

The solution above is always positive. Finally, we put  and  in the required inequality, we reach the following inequality
: $\langle\nabla_{W_k}g - \nabla_{W_k}g, W_k - W_j\rangle \geq \alpha \langle  W_k - W_j,  W_k - W_j \rangle$ only if
\begin{equation}
    \label{eq_32}
    \begin{aligned}
        \alpha \leq \frac{2(c-1)\mu(\mu_1 - \mu_2)}{c^2( \sigma_1^2 + \sigma_2^2 + (\mu_1 - \mu_2)^2 )}.
    \end{aligned}     
\end{equation}
Under a mild assumption that $\alpha^*\in\mathbb{R}$ satisfies \Cref{eq_32}, we conclude that the cross-entropy loss function is $\alpha^*$-strongly convex and $\beta^*$-smooth, such that it is $(\frac{\alpha^*}{2}, \frac{1}{2\beta^*},0)$-correlated with $z^*$. Thus, it meets the \textit{sufficient condition} of convergence.
\end{proof}

\subsection{Proof of Theorem~\ref{theo_linear} for the Dictionary Orthogonality in NAC}
~\label{theo_linear_proof}
\begin{theorem}
Let the neural weights of each operator in a deep neural network be an atom and stack them column-wisely, we can guarantee the existence of an orthogonal dictionary.
\end{theorem}
\begin{proof}
Given a dictionary $\boldsymbol{H}\in R^{n\times K}$, where $n$ is the number of nodes and $K$ is the number of opearators in each layer, we have its mutual coherence computed as follows,
\begin{equation}
    \varphi=\max _{\boldsymbol{h}_{i}, \boldsymbol{h}_{j} \in \boldsymbol{H}, i \neq j}\left|\left\langle\frac{\boldsymbol{h}_{i}}{\left\|\boldsymbol{h}_{i}\right\|_{2}}, \frac{\boldsymbol{h}_{j}}{\left\|\boldsymbol{h}_{j}\right\|_{2}}\right\rangle\right|,
    %= \max _{i \neq j}\left|\cos \theta_{i j}\right|
\end{equation}
where  $\varphi \in [0,1]$, and $\left \langle \cdot  \right \rangle$ denotes inner product.
Here, each atom, $\boldsymbol{h}_{i}$, is the weights from an operator.
The minimum of $ \varphi$ is $0$ and is attained when there is an orthogonal dictionary, while the maximum is $1$ and it attained when there are at least two collinear atoms (columns) in a dictionary.
\eat{The minimum of $ \varphi$ is reached for an orthogonal dictionary, and the maximum for a dictionary containing at least two collinear atoms (columns), which are $0$ and $1$, respectively.}

Let $E_i =  \frac{\boldsymbol{h}_{i}}{\left\|\boldsymbol{h}_{i}\right\|_{2}}$ and $E_j =  \frac{\boldsymbol{h}_{j}}{\left\|\boldsymbol{h}_{j}\right\|_{2}}$,
by Central Limit Theorem~\citep{Fischer2011}, we know that 
$  \left \langle E_i,E_j \right \rangle/\sqrt{n}$ 
converges to a normal distribution, i.e.,
\begin{equation}
    \begin{aligned}
     \left \langle E_i,E_j \right \rangle=\lim _{n \rightarrow \infty} \sqrt{n}Z,
    \end{aligned}
\end{equation}
where  $Z$ is  a standard normal distribution.
Consider  $\bar{E}$ as the mean value of 
all $ \left \langle E_i,E_j \right \rangle$.
With weak law of large numbers (a.k.a. Khinchin's law)~\citep{Khinchin},
for any positive number $\varepsilon$,
the probability that sample average $\bar{E}$ greater than  $\varepsilon$ converges $0$ is written as 
\begin{equation}
    \begin{aligned}
        \lim_{n \rightarrow \infty} \operatorname{Pr} \left(\left|\bar{E}\right|\geq\varepsilon\right)=0
     \end{aligned}
\end{equation}
This implies that the probability that  the inner product of $E_i$ and $E_j$ is greater than $\varepsilon $
close to zero when $n\rightarrow \infty$.  In other words,
the probability that $E_i$ and $E_j$ are nearly orthogonal goes to $1$ when their dimensionality is high. Therefore, the coherence of this dictionary reaches the minimum at a high dimensionality that holds for deep neural networks naturally.
\end{proof}

\subsection{Realization and Details of NAC and NAC-updating}
\label{subsec:complexity}
\begin{figure}[!ht]
    \footnotesize
    \centering
    \includegraphics[trim=0cm 7cm 4cm 0cm,clip=true,width=0.65\textwidth]{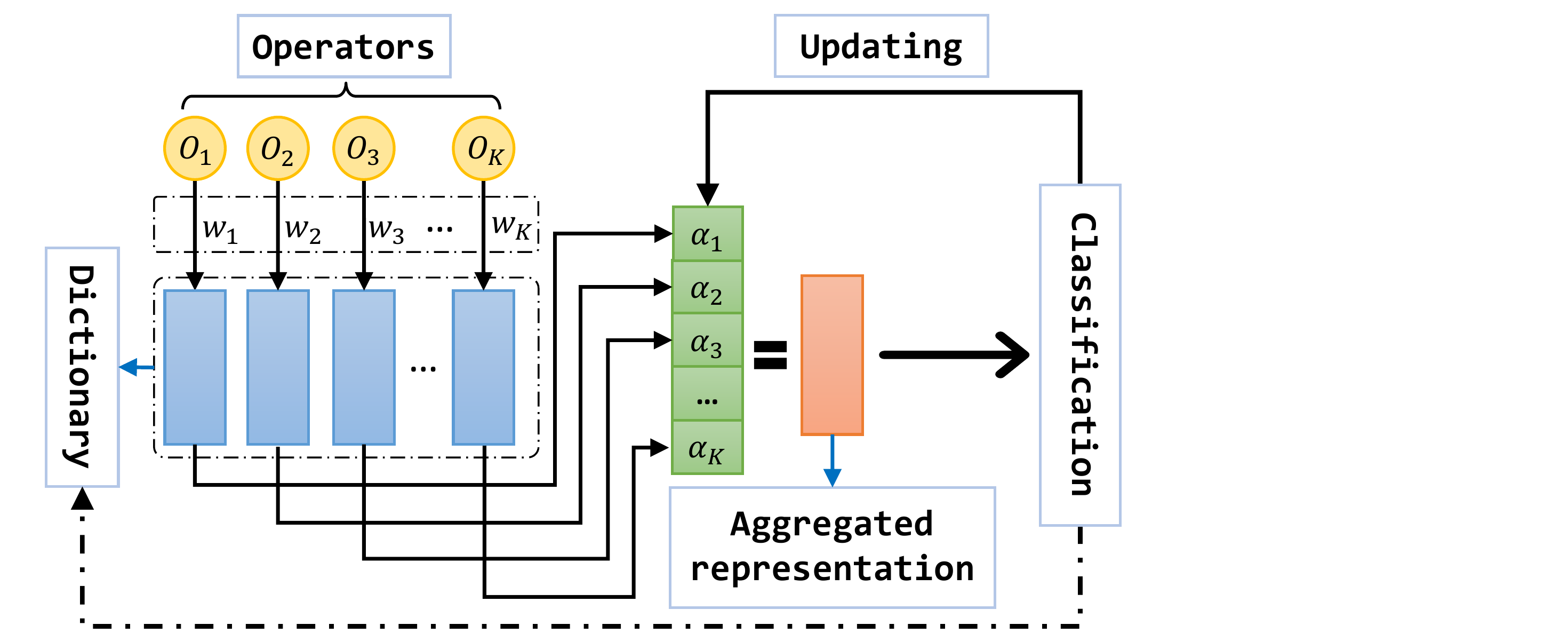}
    \hspace{-2in}
    \caption{
    %\footnotesize
     The framework of the proposed NAC in one layer. NAC  directly learns the architecture $\alpha$ with a fixed dictionary. NAC-updating updates the dictionary in training, showing the additional process with a dash line.}
    \label{fig_nac_framework}
\end{figure}

\begin{center}
\begin{algorithm}[!ht]
\caption{The NAC algorithm}
\begin{algorithmic}[1]
\footnotesize
% \scriptsize
\REQUIRE The search space $\mathcal{A}$;
\ENSURE The architecture $\boldsymbol{\alpha}$\\
Randomly initializing $\bm{W}^l$, for $l=1, \ldots, L$;
set $\boldsymbol{\alpha}=\bf{1}$;
\WHILE{$t = 1, \dots, T$}
    \STATE Performing the feature aggregation at each layer as $\bar{o}^{l}(x) = \bm{o}^{l}\hat{\boldsymbol{\alpha}_{l}} = \sum_{k=1}^{K} \frac{\alpha_{lk}}{|| \bm{\alpha_{l}}||_{2}}o^{lk}(x)$;
    \STATE Computing  $\mathbf{h}_{v}^{l}=\phi\left[ {\bm{W}}^{l} \cdot \bar{o}^{l}\left(\left\{\mathbf{h}_{u}^{l-1}, \forall u \in {N}(v)\right\}\right)\right]$;
    \STATE Optimizing $\boldsymbol{\alpha}$ based on the objective function in \Cref{nac_obj_ce} w.r.t. $\boldsymbol{\alpha}$ under the fixed dictionary $\bm{o}$ ;% as in Eq.7; 
    \STATE Updating $\bm{W}_o$ based on the objective function in \Cref{nac_obj_ce} w.r.t.  $\bm{W}_o$ under fixed $\boldsymbol{\alpha}$;
\ENDWHILE
\STATE Obtain the final architecture $\{\boldsymbol{\alpha}^{*}\}$ from the trained $\boldsymbol{\alpha}$ via an argmax operation at each layer;
\end{algorithmic}
\end{algorithm}
\end{center}

In this work, we propose two realizations of NAC, namely NAC and NAC-updating, where \Cref{fig_nac_framework} outlines the core idea of both methods in one layer. The objectives in \Cref{nac_obj_ce} are optimized by gradient descent, and we therefore omit the details but present the basic pipeline in \Cref{alg_nac1}. The computation of NAC has two major parts: the forward pass and the backward pass. Given the search space, the computation of the forward is then fixed, regarding as a constant. Therefore, the computational complexity mainly focuses on the backward pass in the NAC algorithms.

\textbf{NAC}. The main version of our work has no need to update weights, but only to update architectural parameter $\bm{\alpha}$ during the training process. Therefore, the algorithmic complexity is as  $O(T*\left \| \bm{\alpha} \right \|)$, which is a \textit{linear} function w.r.t $\bm{\alpha}$. The dimension of $\bm{\alpha}$ is often small, which makes the model easy to scale to large datasets and high search space. We call it NAC. %\vsa 

\textbf{NAC-updating}. When updating weights, similar to DARTS, the complexity is estimated as $O(T*(\left \| \bm{\alpha} \right \|  +\left \| w \right \|))$. The dimension of $w$ is often much larger than $\bm{\alpha}$, therefore, the complexity is dominated by updating $w$, where the complexity is  $O(T*\left \| w \right \|)$. Since the dimension of $\bm{\alpha}$ is much smaller than $w$, the complexity of NAC is much less than this version.

\subsection{Experimental Implementation Details}
\label{subsec:impl}
\eat{
\textcolor{red}{\textbf{Environments}. Add a brief description of the software env and hardware. For instance, All the experiments were performed on an Amazon EC2 cluster with an Intel Xeon E5-2686 v4 CPU, 61GB DRAM, and an NVIDIA Tesla V100 GPU}
}

\textbf{Searching configuration.} 
% \todo{We follow the settings in SANE~\citep{zhao2021search} as ours.}
In our experiments we adopt $3$-layer GNN as the backbone.
Unless specified, our experiments follow the same settings for searching architectures as SANE~\citep{zhao2021search} :
\begin{compactitem}[-]
    \item $Architecture\ optimizer.$ We use Adam for training the architecture parameters $\bm{\alpha}$. We set the learning rate as $0.0003$ and  the weight decay as $0.001$.   Also, the $\beta_{1}$ and $\beta_{2}$  are fixed as  $0.5$ and $0.999$, respectively.
    All runs a  constant schedule for training, such as $100$ epochs.
    % The learning rate has a constant schedule for $100$ epochs for all. 

    \item $Weight\ optimizer.$ We use SGD to update models' parameters, i.e., $w$. The learning rate and SGD momentum are given as $0.025$ and $0$, respectively, where the learning rate has a cosine decay schedule for $100$ epochs. We fix the weight decay value, i.e. set $\rho_1=0.0005$.
    
    \item $Batch \ size.$ For transductive tasks, we adopt  in-memory datasets, and the $batch\_size$ is fixed as the size of the dataset themselves. 
\end{compactitem}

\eat{
\begin{table}[!hbt]
    \centering
    \caption{ The performance comparison of our methods under different initialization strategies, including normal, uniform and orthogonal initialization. All three initializations yield close performance (i.e., below 1\% difference), showing that NAC is robust to initialization. Among all, orthogonal initialization shows the best performance and considered as the default choice of NAC.}
    \label{Table_different_inital_app}
    \footnotesize
    \begin{tabular}{cl|cccccc}
        \toprule
        \multirow{1}{*}{} &~&
        \multicolumn{1}{c}{CiteSeer} &
        \multicolumn{1}{c}{Cora} &
        \multicolumn{1}{c}{PubMed} \\
        ~ & ~& {Acc(\%)} & {Acc(\%)}  & {Acc(\%)}  \\
        \midrule
        \multirow{2}{*}{Normal} & LNAC          & \textbf{74.70} &\textbf{87.96}& \textbf{88.56}\\
        & XNAC          & 74.09 & 87.59& 85.50\\
        \midrule 
        \multirow{2}{*}{Uniform} & LNAC    & 73.64 &\textbf{88.52}& \textbf{89.01}\\
        & XNAC          & \textbf{74.40} &87.04& 88.20 \\
        \midrule
        \multirow{2}{*}{Orthogonal} & LNAC          & \textbf{75.00} &\textbf{88.33}& \textbf{89.01}\\
        & XNAC          & 74.40 &87.59& 88.94 \\
        \bottomrule
    \end{tabular}
\end{table}
}

\noindent \textbf{The configuration for retraining phase.}
At the retraining state, we adopt Adam as the optimizer and set the scheduler with cosine decay to adjust the learning rate. The total number of epochs is fixed  $400$ for all methods for fairness. Please refer to the setting of SANE~\citep{zhao2021search} and EGAN~\citep{zhao2021efficient} for more details as we follow this in our experiment.

For CiteSeer dataset, we set the  initial learning rate as $0.005937$ and weight decay as $0.00002007$. The configuration for models is as follows: $hidden\_size=512$, $dropout=0.5$, and using $ReLU$ as the activation function.

For Cora dataset, we set the initial learning rate as $0.0004150$, and weight decay as $0.0001125$. In model, we set $hidden\_size = 256$ , $dropout=0.6$, and use $ReLU$ as the activation function.

For PubMed dataset, we set the initial learning rate as $0.002408$ and weight decay as $0.00008850$. As for the model, we have $hidden\_size=64$ and $dropout=0.5$, and use $ReLU$ as the activation function. %\vspace{-0.3cm}

For Amazon dataset, we set the initial learning rate as $0.002111$ and weight decay as $0.000331$. As for the model, we have $hidden\_size=64$ and $dropout=0.5$, and use $elu$ as the activation function. %\vspace{-0.3cm}

For PPI dataset, we set the initial learning rate as $0.00102$ and weight decay as $0$. As for the model, we have $hidden\_size=512$ and $dropout=0.5$, and use $Relu$ as the activation function.

For ognb-arXiv data, due to the limitation of computational resources, we use the NAS-Bench-Graph\cite{qin2022bench} to get the performance of searched architecture and make some adjustments to align with the supporting search space in NAS-Bench-Graph.

\noindent \textbf{Solving $\normlone$ regularization.}
The $\normlone$ regularization, also known as \textbf{Lasso Regression} (Least Absolute Shrinkage and Selection Operator), adds an absolute value of the magnitude of the coefficient as a penalty term to the loss function \citep{ranstam2018lasso}.
% The $\normlone$ regularization is also called \textbf{Lasso Regression} (Least Absolute Shrinkage and Selection Operator), which adds an absolute value of magnitude of coefficient as penalty term to the loss function. 
Using the $\normlone$ regularization, the coefficient of the less important feature is usually decreased to zero, sparsifying the parameters. It should be noted that since $||\bm{\alpha}||_1$ is not differentiable at $\bm{\alpha} = \bm{0}$, the standard gradient descent approach cannot be used.
% The coefficient of the less significant feature is often shrunk to zero using the $\normlone$ regularization, sparsifying the parameters. It should be noted that the typical gradient descent approach cannot be used since $||\bm{\alpha}||_1$ is not differentiable at $\bm{\alpha} = \bm{0}$.
\\
Despite the fact that the loss function of the Lasso Regression cannot be differentiated, many approaches to problems of this kind, such as \citep{schmidt2009optimization}, have been proposed in the literature. These methods can be broadly divided into three groups: constrained optimization methods, unconstrained approximations, and sub-gradient methods.
\\
Since subgradient methods are a natural generalization of gradient descent, this type of methods can be easily implemented in Pytorch's framework. Lasso Regression can be solved using a variety of subgradient techniques; details on their implementation can be found in \citep{fu1998penalized} and \citep{shevade2003simple}.
% There are a number of Subgradient methods aimed to solving Lasso Regression, and specific implementation details can be found in \citep{fu1998penalized} and \citep{shevade2003simple}.

\noindent \textbf{Computational Complexity Estimation of NAC.}
The computation of NAC has two major parts: the forward pass and the backward pass. Given the search space, the computation of the forward is then fixed and regarded as a constant. Therefore, the computational complexity mainly focuses on the backward pass in the NAC algorithms.\\
The main version of our work does not need to update weights, but only to update architectural parameter $\bm{\alpha}$ during the training process. 
Therefore, the algorithmic complexity is as  $O(T*\left \| \bm{\alpha} \right \|)$, which is a \textit{linear} function w.r.t $\bm{\alpha}$. The dimension of $\bm{\alpha}$ is often small, which makes the model easy to scale to large datasets and high search space. When updating weights of the linear layer, the complexity is estimated as $O(T*(\left \| \bm{\alpha} \right \|  +\left \| \bm{W}_{o} \right \|))$. The dimension of $\bm{W}_{o}$ is a constant number, that equals the number of classes. Therefore, the complexity is almost the same as the main version of NAC, where the complexity is  $O(T*\left \| \bm{\alpha} \right\| + \left \| \bm{W}_{o} \right\|)$. \\
When updating weights, similar to DARTS, the complexity is estimated as $O(T*(\left \| \bm{\alpha} \right \|  +\left \| \bm{w} \right \|))$. The dimension of $\bm{w}$ is often much larger than $\bm{\alpha}$, therefore, the complexity is dominated by updating $\bm{w}$, where the complexity is  $O(T*\left \| \bm{w} \right \|)$. Since the dimension of $\bm{\alpha}$ is much smaller than $\bm{w}$, the complexity of NAC is much less than this type of methods. 

\noindent \textbf{Approximate Architecture Gradient.}
Our proposed theorems imply an optimization problem with $\bm{\alpha}$ as the upper-level variable and $\bm{W}_{o}$ as the lower-level variable:
\begin{equation}~\label{nac_obj_new}
\begin{aligned}
\begin{cases}
\bm{\alpha}^{*} =\underset{\bm{\alpha}}{\operatorname{argmax}}\mathcal{M} \left(\bm{W}_{o}^{*}(\bm{\alpha}), \bm{\alpha}\right)\\ 
\bm{W}_{o}^{*}(\bm{\alpha}) =\underset{\bm{W}_{o}}{\operatorname{argmin}} \mathcal{L}(\bm{\alpha}, \bm{W}_{o}),
\end{cases}
\end{aligned}
\end{equation}
Following ~\citep{liu2019darts}, we can adopt a First-order Approximation to avoid the the expensive inner optimization, which allows us to give the implementation in the \Cref{alg_nac1}.

\subsection{Ablation Studies}
\label{subsec:ablation}

\subsubsection{The Effect of Sparsity}~\label{subsubsec:theeffectofsparsity}
Our model uses the hyperparameter $\rho$ to control the sparsity of the architecture parameter $\bm{\alpha}$, where a large sparsity is to enforce more elements to be zero. 
We investigate the effect of this hyperparameter by varying its value in a considerably large range, such as $\left [ 0.001,10 \right ]$.
We first present the accuracy of different sparse setting in \Cref{fig_sparse_sensitivity}. We find that the results vary little in a considerably wide range, this indicates our models are insensitive to sparsity hyperparameter in general.
\eat{
In \Cref{fig_sparse_nac_searchspace}, we present the search quality under different sparsity levels for each dictionary. The resulting distribution is largely overlapped, which provides another perspective to understand the robustness of our model of this parameter.
This property leaves us decent flexibility in setting this hyperparameter without comprising the performance.
}

\begin{figure*}[!ht]
	\centering
	\subfigure[\small{CiteSeer}]{\includegraphics[scale=0.35]{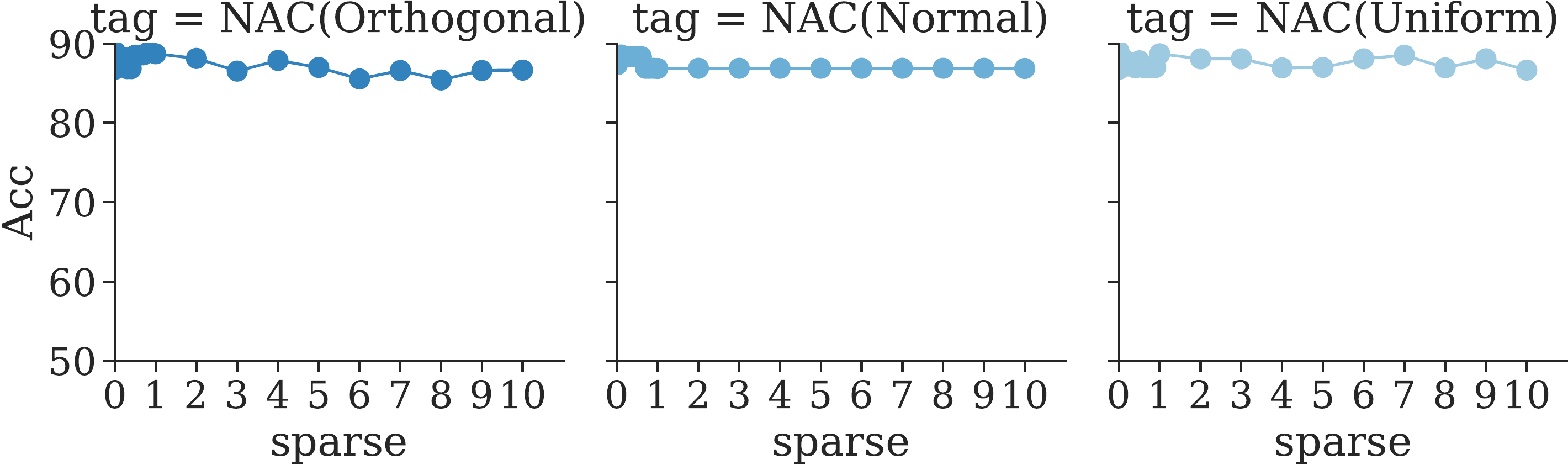}}
	\subfigure[\small{Cora}]{\includegraphics[scale=0.35]{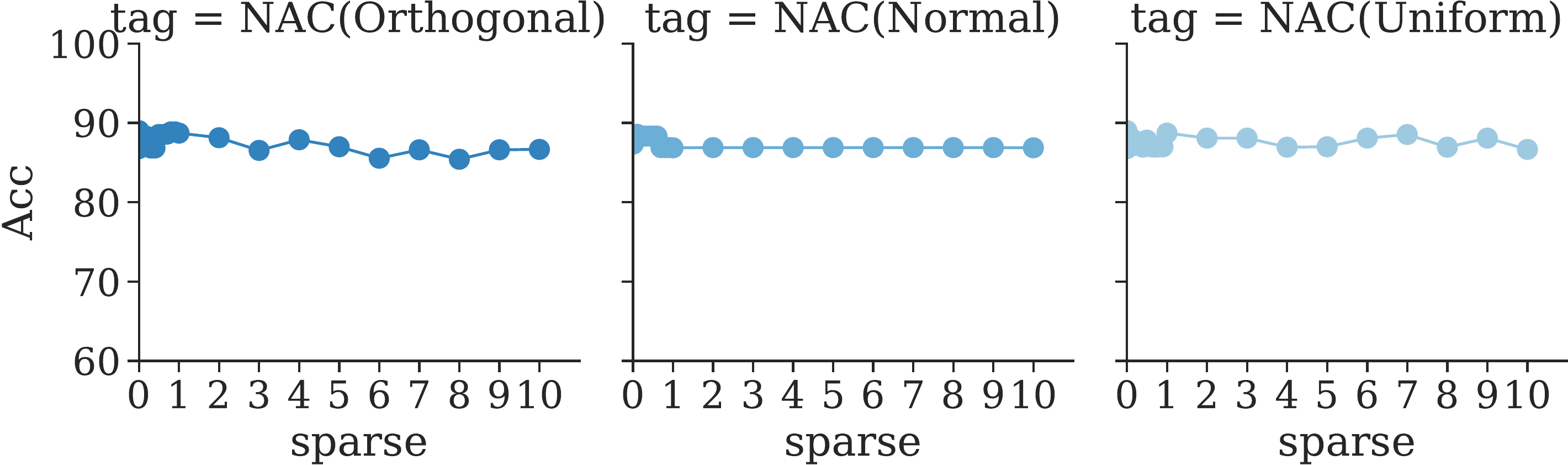}}
	\subfigure[\small{PubMed}]{\includegraphics[scale=0.35]{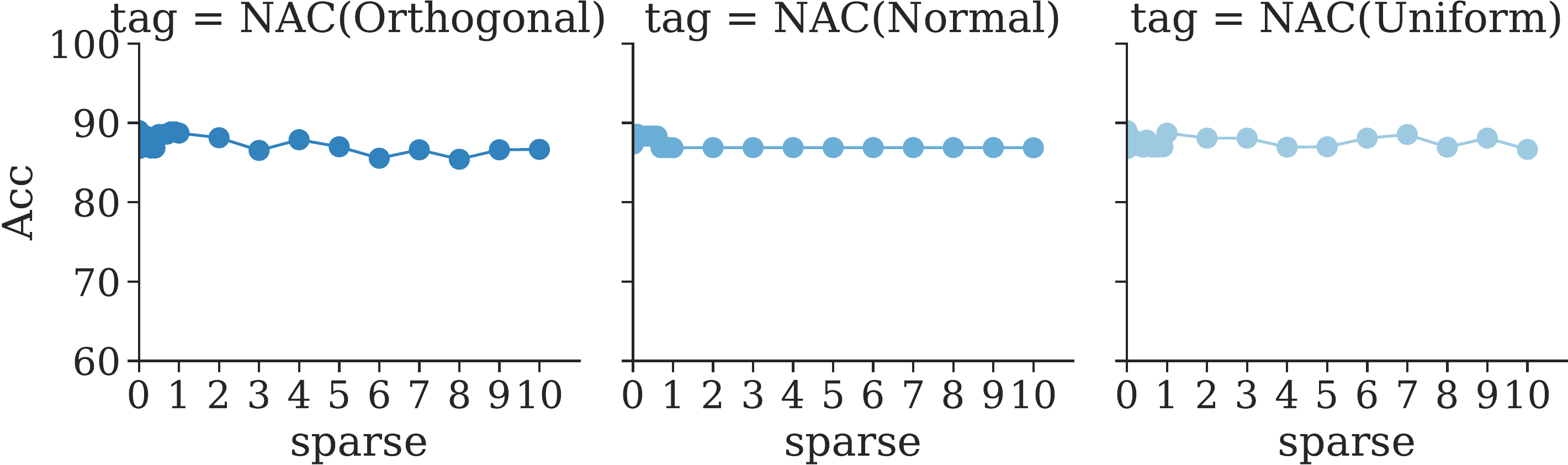}}
	\caption{
	Sensitivity study of the sparsity. Results are with varying sparsity (x-axis) and different initialization NAC methods (i.e., normal, uniform and orthogonal). The variation is small, showing the robustness of our model w.r.t the sparsity.
	}
	\label{fig_sparse_sensitivity}
\end{figure*}

\subsubsection{The Effect of Random Seeds}~\label{subsubsec:randomSeeds}
Random seeds often play an importance role in traditional NAS methods as it affects the initialization significantly. 
People often report the average or the best results under different random seeds, this may lead to poor reproducibility.
To the best our knowledge, this is for the first we explicitly demonstrate the effect of random seeds in this subject.
We run experiments on several random seeds and report the results of NAC on Pubmed dataset, as shown in \Cref{fig:random_seeds}. 
In particular, we implement multiple combinations of random seeds and sparsity to observe the variation on performance. Note that we round the values to integer to fit the table.
In all these combinations, we have the average and variance as $87.32\%$ and $0.9\%$, respectively. 
The average performance is comparable to the best results from all competitive results, which indicates the stability of NAC.
% \todo{ check numbers }

\begin{figure}[!ht]
    \centering
    \includegraphics[width=.35\textwidth]{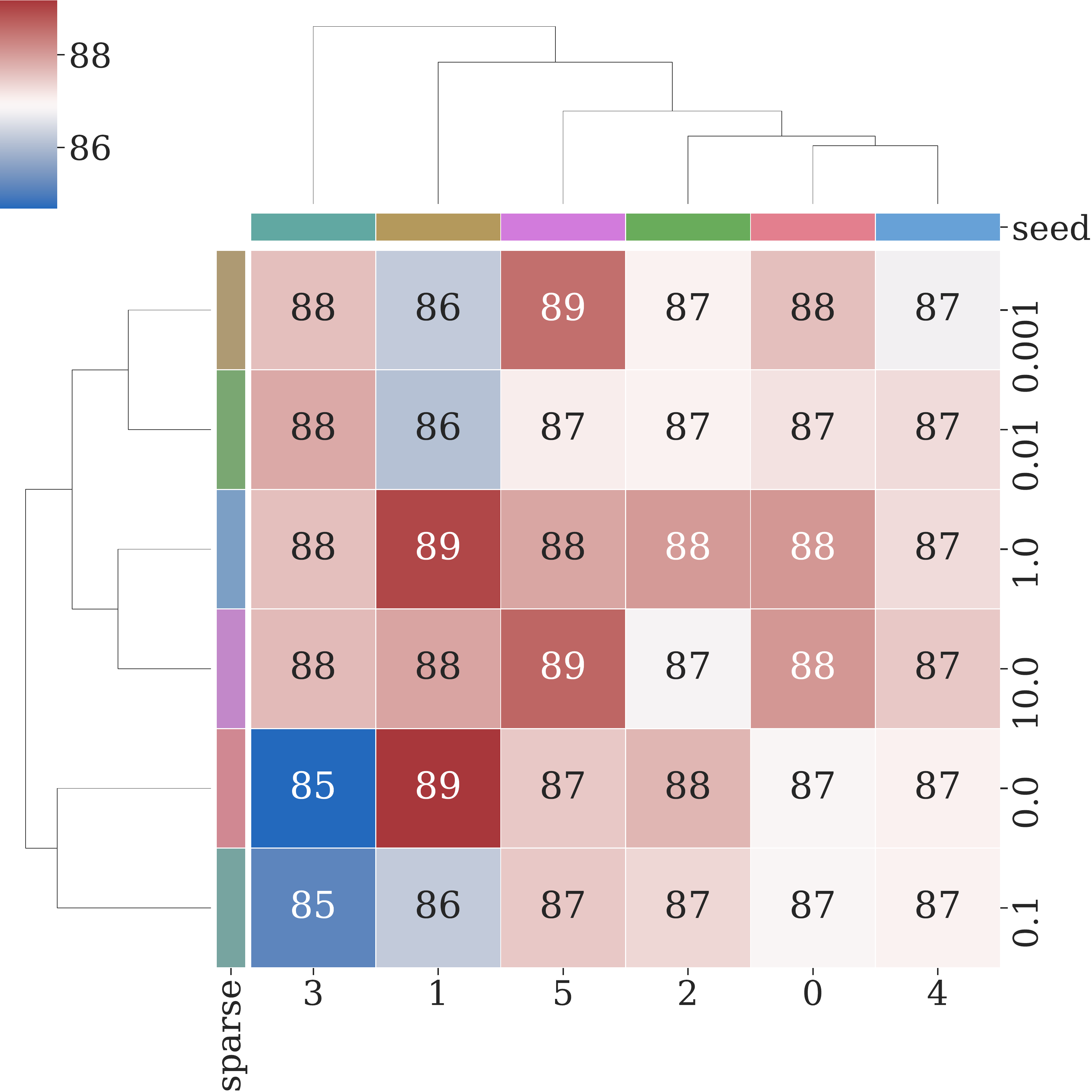}
    \caption{The effects of random seeds of NAC on the accuracy, where x-axis denotes the random seeds and y-axis denotes the sparsity. NAC performs stably with random seeds.}
    \label{fig:random_seeds}
\end{figure}

\subsubsection{The Effect of Training on The Final Linear Layer}
\label{subsubsec:traininglinear}
Our proposed theorems prove that a GNN with randomly initialized weights can make the final output as good as a well-trained network when initializing networks with orthogonal weights and updating the total network using gradient descent. In practice, we find it difficult to determine at what training epoch the optimal weight parameters can be obtained through training linear layer. We noticed that most of the time, the untrained weights in the initial state can often already exceed the accuracy that can be obtained from the weights after multiple epochs of training the final linear layer,  as shown in \Cref{fig:random_linear_updating}.  Therefore, we further omit the training of the final linear layer. It is important to note that this approximation is based on our proposed theorems in which most of the intermediate layers do not require training.

\begin{figure}[!ht]
    \centering
\includegraphics[width=.45\textwidth]{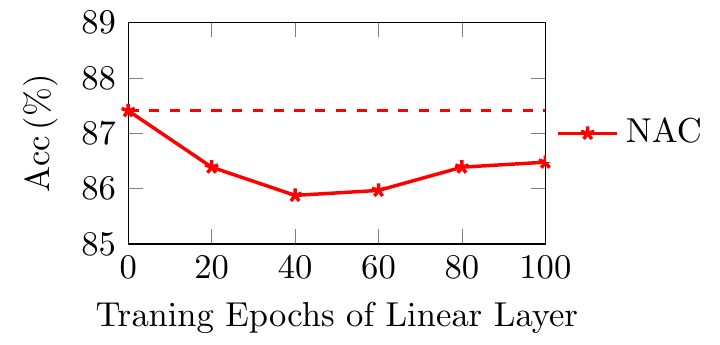}
    \caption{
        The effects of training final linear layer of NAC on the accuracy, where x-axis denotes the training epochs of the final linear layer and y-axis denotes the averaged accuracy of acquired architecture $\bm{\alpha}$ using the corresponding weights.
    }
    \label{fig:random_linear_updating}
\end{figure}

\subsection{Runtime of each method on a single GPU Server}
\eat{To align the platforms used for the speed measurements, we performed all of the speed measurement experiments for the compared methods and our proposed method on the same GPU platform. Speed measurement experiments were measured using PyTorch (version 1.10.2+gpu) on a GPU server with four NVIDIA 3090 GPUs (24G). The experimental results are concluded in \Cref{Table_Speed_GPU}.}
{Apart from the running time on the CPU, we also measure the running time for all methods on a GPU platform, where we use PyTorch (version 1.10.2+gpu) on a GPU server with four NVIDIA 3090 GPUs (24G), as shown in Table \Cref{Table_Speed_GPU}.
These results are consistent with the ones in Table 1, demonstrating our advantage in speed.
}

\begin{table}[!ht]
\centering
\caption{\footnotesize Timing results of the compared methods and our NAC method in one the same GPU server. NAC attains superior performance in efficiency (in seconds).}
% \todo{change hours to mints}.  

\footnotesize
\begin{tabular}{l|cccc}
\toprule
\multirow{1}{*}{} &
\multicolumn{1}{c}{CiteSeer} &
\multicolumn{1}{c}{Cora} &
\multicolumn{1}{c}{PubMed} &
\multicolumn{1}{c}{ Computers}
\\
\midrule
RS & 196.00 & 328.00 & 461.00 & 900.00\\ 
BO & 225.00 & 355.00 & 462.00 & 878.00\\
GraphNAS & 6193.00 & 6207.00 & 6553.00 & 8969.00\\
GraphNAS-WS  & 947.00 & 1741.00 & 2325.00 & 4343.00  \\
SANE & 35.00 & 41.00 & 43.00 & 43.00\\ \midrule
NAC & \textbf{14.00} & \textbf{14.00} & \textbf{15.00} & \textbf{14.00} \\
NAC-updating & 42.00 & 31.00 & 36.00 & 42.00\\ 
\bottomrule
\end{tabular}
\label{Table_Speed_GPU}
\end{table}

\subsection{Performance on Inductive Tasks}
~\label{appendix:exp_ppi}
To validate the effectiveness of our approach to the Inductive task, we performed a set of experiments on the PPI dataset. The experimental results are concluded in \Cref{Table_PPI}. The experimental results show that our proposed method effectively outperforms the best baseline method by about 4\% in terms of Micro-F1 score. Besides, our method achieves an $8\times$ speedup in terms of running time than SANE. Also, the non-updating scheme of the NAC approach exceeds the NAC-updating method effectively.\\
Experiments on the PPI dataset further validate the effectiveness and superiority of our proposed method.

\begin{table}[!ht]
\centering
\caption{\footnotesize Experimental results on the compared methods: our NAC attains superior performance on PPI dataset in both Micro-F1 score (\%) and efficiency (in hours).}

\footnotesize
\begin{tabular}{l|cccc}
\toprule
\multirow{1}{*}{} &
\multicolumn{1}{c}{PPI(Micro-F1(\%))} &
\multicolumn{1}{c}{PPI(Time(h))} &
\\
\midrule
SANE & 91.01${_{\pm6.83}}$ & 2.50 \\
NAC & \textbf{95.16${_{\pm0.03}}$} & \textbf{0.31}  \\
NAC-updating & 94.47${_{\pm7.09}}$ & 4.68 \\ 
\bottomrule
\end{tabular}
\label{Table_PPI}
\end{table}

\subsection{Performance on ogbn-arXiv}~\label{appendix:exp_ognb}
To validate the effectiveness of our approach to the Inductive task, we performed a set of experiments on the ogbn-arXiv dataset. The experimental results are concluded in \Cref{Table_ogbn}. The experimental results show that our proposed method effectively outperforms the best baseline method by about 4\% in terms of Micro-F1 score. Besides, our method achieves an $3\times$ speedup in terms of running time than SANE. Also, the non-updating scheme of the NAC approach exceeds the NAC-updating method effectively.\\
Experiments on the ogbn-arXiv dataset further validate the effectiveness and superiority of our proposed method.
\begin{table}[!ht]
\centering
\caption{Experimental results on the compared methods: our NAC attains superior performance on ogbn-arXiv dataset in Accuracy (\%).}
\footnotesize
\begin{tabular}{l|cccc}
\toprule
\multirow{1}{*}{} &
\multicolumn{1}{c}{ogbn-arXiv(Accuracy(\%))} &
\multicolumn{1}{c}{ogbn-arXiv(Time(min))}
\\
\midrule
SANE & 70.94${_{\pm0.85}}$ & 1.53 \\
NAC & \textbf{71.13${_{\pm0.59}}$} & \textbf{0.50} \\
NAC-updating & 71.10${_{\pm0.47}}$ & 2.20\\ 
\bottomrule
\end{tabular}
\label{Table_ogbn}
\end{table}

\subsection{Searched Architectures for Each Dataset of our method}
We visualize the searched architectures (top-1) by NAC on different datasets in \Cref{fig_architecutres}.
\begin{itemize}
    \item For Citeseer dataset, the searched result is [GAT, GCN, Chebconv];
    \item For Cora dataset, the searched result is [GIN, GIN, GCN];
    \item For Pubmed dataset, the searched result is [GCN, GAT\_Linear, Geniepath;
    \item For Amazon Computers dataset, the searched result is [Geniepath, GCN, SAGE];
    \item For PPI dataset, the searched result is [Chebconv, GAT\_COS, Chebconv];
\end{itemize}

\begin{figure*}[!hb]
    \centering
    \begin{minipage}[b]{0.19\textwidth}
     \captionsetup{font={small}}
		\centering
		\includegraphics[scale=0.70]{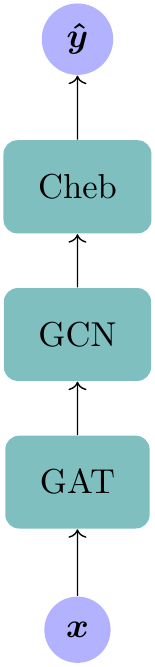}
		\centerline{\small{(a) CiteSeer}}
	\end{minipage}
    % \hspace{.08in}
	\begin{minipage}[b]{0.19\textwidth}
        \captionsetup{font={small}}
		\centering
		\includegraphics[scale=0.70]{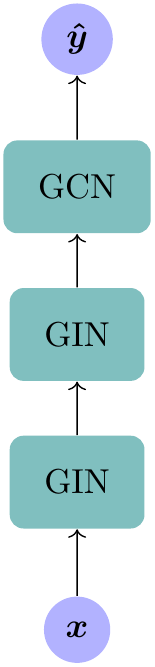}
		\centerline{\small{(b) Cora}}
	\end{minipage}
    % \hspace{.08in}
	\begin{minipage}[b]{0.19\textwidth}
        \captionsetup{font={small}}
		\centering
		\includegraphics[scale=0.70]{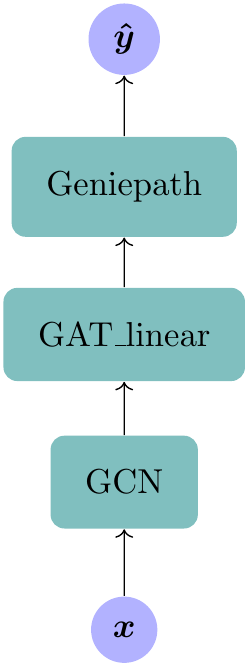}
		\centerline{\small{(c) PubMed}}
	\end{minipage}
	\begin{minipage}[b]{0.19\textwidth}
        \captionsetup{font={small}}
		\centering
		\includegraphics[scale=0.70]{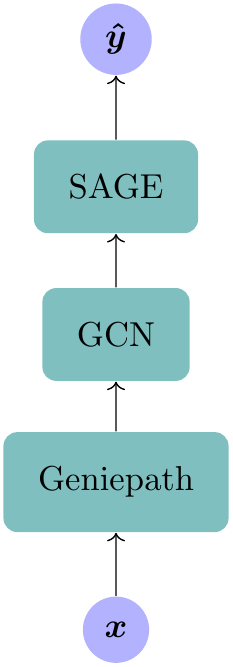}
		\centerline{\small{(d) Amazon Computers}}
	\end{minipage}
	\begin{minipage}[b]{0.19\textwidth}
        \captionsetup{font={small}}
		\centering
		\includegraphics[scale=0.70]{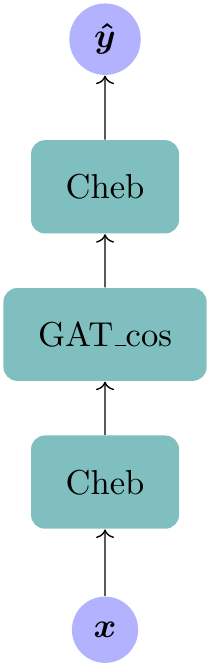}
		\centerline{\small{(e) PPI}}
	\end{minipage}
	\caption{The searched architectures of NAC on benchmark datasets.}
	\label{fig_architecutres}
\end{figure*}

\subsection{Related Work}
\label{appendix:related_work}
Another line of work~\citep{chen2021neural}~\citep{shu2021nasi}  uses neural tangent kernels (NTK) to search the network structure, called training-free NAS and focuses on  CNN architectures. 
These hold strong assumptions when analyzing due to the need for infinite width of networks, and thus far from reality. In contrast, we do not have such an assumption by taking advantage of the built-in linearity in GNNs to get untrained GNNs to work.

% \subsection{REPRODUCIBILITY STATEMENT}
% The supplemental material includes the code for our experiments.

% You can have as much text here as you want. The main body must be at most $8$ pages long.
% For the final version, one more page can be added.
% If you want, you can use an appendix like this one, even using the one-column format.
% %%%%%%%%%%%%%%%%%%%%%%%%%%%%%%%%%%%%%%%%%%%%%%%%%%%%%%%%%%%%%%%%%%%%%%%%%%%%%%%
% %%%%%%%%%%%%%%%%%%%%%%%%%%%%%%%%%%%%%%%%%%%%%%%%%%%%%%%%%%%%%%%%%%%%%%%%%%%%%%%
% \fi

\end{document}